\documentclass[11pt,letterpaper]{article}
\usepackage[hypertexnames=false]{hyperref}
\hypersetup{
     colorlinks   = true,
     linkcolor = black,
     urlcolor    = teal,
	 citecolor = teal
}
\usepackage{authblk}
\usepackage{natbib}
\usepackage[margin=1in]{geometry}
\usepackage[utf8]{inputenc} % allow utf-8 input
\usepackage[T1]{fontenc}    % use 8-bit T1 fonts
\usepackage{url}            % simple URL typesetting
\usepackage{booktabs}       % professional-quality tables
\usepackage{amsfonts}       % blackboard math symbols
\usepackage{nicefrac}       % compact symbols for 1/2, etc.
\usepackage{microtype}      % microtypography
\usepackage{xcolor}         % colors

\usepackage{amsmath,amsfonts,amssymb}
\usepackage{aliascnt}
\usepackage{amsthm}
\usepackage{bm}
\usepackage{enumitem,array,booktabs}
\usepackage{algorithm}
\usepackage{algorithmic}
\newcounter{procedure}

\usepackage[capitalize]{cleveref}

\usepackage{natbib}
\usepackage{bbm}
\usepackage{tabulary,multirow}
\usepackage{makecell}
\usepackage{dsfont}
\usepackage{tcolorbox}
\usepackage{thmtools}
\usepackage{xpatch}
\xpatchcmd{\proof}{\itshape}{\scshape}{}{}
\usepackage{tablefootnote}
\definecolor{theoremcolor}{HTML}{F2F2F2}
\usepackage{mdframed}
\newtheorem{prop}{Proposition}
\newtheorem{lemma}{Lemma}

\newcommand{\floor}[1]{\left\lfloor#1\right\rfloor}

\newcommand{\R}{\mathbb{R}}

\newcommand{\NN}{{\mathbb N}}

\newcommand{\E}{\mathbb{E}}

\newcommand{\EEs}[2]{\underset{#1}{\mathbb{E}}\left[#2\right]}

\newcommand{\ac}[1]{\left\{#1\right\}}

\newcommand{\norm}[1]{\left\|#1\right\|}

\newcommand{\cA}{\mathcal{A}}

\newcommand{\cD}{\mathcal{D}}

\newcommand{\F}{\mathcal{F}}
\newcommand{\cF}{\mathcal{F}}

\newcommand{\cH}{\mathcal{H}}

\newcommand{\cM}{\mathcal{M}}

\newcommand{\cX}{\mathcal{X}}

\renewcommand{\epsilon}{\varepsilon}
\renewcommand{\hat}{\widehat}
\renewcommand{\tilde}{\widetilde}
\renewcommand{\bar}{\overline}

\newcommand{\bmu}{{\boldsymbol \mu}}

\newcommand{\nothere}[1]{}

\newcommand{\pd}{\mathrm{fat}}

\newcommand{\sign}{\mathrm{sign}}

\newcommand{\sev}{S_{\text{eval}}}

\newcommand\blfootnote[1]{%
  \begingroup
  \renewcommand\thefootnote{}\footnote{#1}%
  \addtocounter{footnote}{-1}%
  \endgroup
}

\date{}
\begin{document}

\title{Scale-Sensitive Shattering: \\Learnability and Evaluability at Optimal Scale}
\author[1]{Shashaank Aiyer}
\author[2]{Yishay Mansour}
\author[3]{Shay Moran}
\author[1]{Han Shao}
\author[4]{Tom Waknine}
\affil[1]{University of Maryland}
\affil[2]{Tel Aviv University and Google Research}
\affil[3]{Technion and Google Research}
\affil[4]{Technion}

\maketitle
\blfootnote{Authors are ordered alphabetically.}
\blfootnote{Emails: \href{mailto:saiyer1@umd.edu}{saiyer1@umd.edu}, \href{mailto:mansour.yishay@gmail.com}{mansour.yishay@gmail.com}, \href{mailto: smoran@technion.ac.il}{ smoran@technion.ac.il}, \href{mailto:hanshao@umd.edu}{hanshao@umd.edu}, \href{mailto:tom.waknine@campus.technion.ac.il}{tom.waknine@campus.technion.ac.il}}

\begin{abstract}
% We study the optimal scales at which real-valued function classes exhibit uniform convergence and learnability. Our main result establishes a scale-sensitive generalization of the fundamental theorem of PAC learning: for every bounded real-valued class $\mathcal F$ and every $\gamma>0$, uniform convergence at scale $\gamma$, agnostic learnability at scale $\gamma/2$, and finiteness of the fat-shattering dimension at every scale $\gamma'>\gamma$ are equivalent. While the qualitative equivalence between agnostic learnability at scale $\gamma/2$ and finite fat-shattering dimension at scale $\gamma$ was previously obtained by Long (COLT 2001), our result gives sharper quantitative bounds: if $d=\mathrm{fat}_{\gamma}(\mathcal F)$, then agnostic learning at scale $\gamma/2$ is possible with sample complexity $\tilde O(d/\varepsilon^2)$, improving Long’s $\tilde O(d/\varepsilon^3)$ bound and matching the optimal dependence on $\varepsilon$ up to logarithmic factors.

We study the optimal scales at which real-valued function classes exhibit uniform convergence and learnability. Our main result is a scale-sensitive generalization of the fundamental theorem of PAC learning: for every bounded real-valued class $\mathcal F$ and every $\gamma>0$, uniform convergence at scale $\gamma$, agnostic learnability at scale $\gamma/2$, and finiteness of the fat-shattering dimension at every scale $\gamma'>\gamma$ are equivalent. This extends Long’s (COLT '01) equivalence between agnostic learnability at scale $\gamma/2$ and fat-shattering at scale $\gamma$ by adding uniform convergence at the corresponding optimal scale. Moreover, if $d=\mathrm{fat}_{\gamma}(\mathcal F)$, our sample complexity is $\tilde O(d/\varepsilon^2)$, improving Long’s $\tilde O(d/\varepsilon^3)$ bound and matching the optimal dependence on $\varepsilon$ up to logarithmic factors.

The key technical ingredient is a direct bound on empirical $\ell_\infty$ covering numbers, avoiding the standard detour through packing numbers. As a consequence, we obtain sharp asymptotic metric-entropy bounds in terms of the fat-shattering scale~$\gamma$: an $O(\log^2 n)$ bound holds already at scale $\gamma/2$, while an $O(\log n)$ bound holds at scale $2\gamma$. We further show that the $O(\log^2 n)$ bound is sometimes tight. These results settle the precise scale relationship between empirical covering numbers and fat-shattering dimension, and resolve questions raised and studied by Alon et al.~(JACM ’97), Anthony and Bartlett (Cambridge Univ.\ Press ’99), and Rudelson and Vershynin (Ann.\ of Math.~’06). 
%They also sharpen the scale dependence in the bounds of Bartlett and Long (JCSS ’98).

As an application, we establish a sharp dichotomy for bounded integral probability metrics: every such IPM is either estimable or cannot be weakly evaluated within any multiplicative factor $c<3$, while $3$-weak evaluability always holds, resolving an open question from Aiyer et al.~(ICML ’26). We also highlight several open questions on quantitative sample complexity and evaluability.

 \end{abstract}

\section{Introduction}
One of the central results in statistical learning theory is the fundamental theorem of PAC learning, which establishes an equivalence between three different notions: agnostic learnability, uniform convergence, and finiteness of the VC dimension. In the binary classification setting, this theorem provides a complete characterization of when a class of functions can be learned from data in a distribution-free manner, and serves as a cornerstone of learning theory~\cite{Vapnik1971ChervonenkisOT,VC-lrn}.

A natural extension of this framework considers real-valued function classes $\cF$, where predictions are no longer binary but continuous. In this setting the VC dimension is replaced by a scale-sensitive extension called the fat-shattering dimension (defined below). This notion was introduced in learning theory by Kearns and Schapire~\cite{KEARNS1994}, and also appears in approximation theory (see Lorentz~\cite{Lorentz1986}, and Tikhomirov~\cite{Tikhomirov1960} who attributes it to earlier work of Kolmogorov).

Understanding the relationship between learnability, fat-shattering, and uniform convergence in the real-valued setting has been a long-standing goal; see, in particular, the detailed discussion in Anthony and Bartlett~\cite[Section~12.7]{Anthony_Bartlett_1999}. The classical no-free-lunch argument shows that if the $\gamma$-fat-shattering dimension is unbounded, then the class is not learnable at scale $\gamma/2$ and does not satisfy uniform convergence at scale $\gamma$, reflecting a scale-sensitive version of the no-free-lunch principle. In the opposite direction, Alon et al.~\cite{Alon1997-scale-sensitive} showed that finite $\gamma$-fat-shattering dimension implies uniform convergence, and hence learnability, at scale $c\gamma$ for some universal constant $c\geq 10$. Bartlett and Long~\cite{BARTLETT1998} improved this implication to every constant $c>2$, and Long~\cite{Long01} later showed that the learnability implication holds at the optimal scale: finite $\gamma$-fat-shattering dimension suffices for agnostic learnability at every scale strictly larger than $\gamma/2$.

The sample-complexity bounds were even less complete. At larger scales, Bartlett and Long obtained the standard $\varepsilon^{-2}$ dependence on the excess error parameter $\varepsilon$, up to logarithmic factors; these logarithmic losses were later removed by~\cite{rudelson2006combinatorics,COLOMBONI2025-improved-UC} via sharp asymptotic entropy bounds. At the optimal learnability scale, however, the only available bound was Long’s, with the weaker dependence $\tilde O(\varepsilon^{-3})$. Thus, Long’s result closed the qualitative scale gap for learnability, but left open both the corresponding optimal-scale question for uniform convergence and the sharp quantitative dependence for learning at that scale.

Our first main result closes the remaining optimal-scale question for uniform convergence and improves the quantitative dependence for learning to near-optimal. To state the result, we briefly recall the relevant notions. Let $\cX$ be a domain, let $\cF \subset [-R,R]^{\cX}$ be a class of bounded real-valued functions \footnote{Throughout the paper, we assume that $
\F$ satisfies some mild measurability conditions such as \cite[Appendix C]{Pollard1984ConvergenceOS}} over $\cX$, and let $\gamma>0$. We say that $\cF$ satisfies $\gamma$-uniform convergence if for every $\varepsilon,\delta>0$ there exists $m$ such that for every distribution $q$ over $\cX$, with probability at least $1-\delta$ over a sample $S \sim q^n$ of size $n\ge m$,
\[
\sup_{f\in\cF}\left|\E_{x\sim q}[f(x)]-\frac{1}{n}\sum_{x\in S} f(x)\right| \le \gamma + \varepsilon.
\]
We say that $\cF$ is $\gamma$-learnable if there exists a learner $\cA$ (that is, a mapping from samples to hypotheses) such that for every $\varepsilon,\delta>0$ there exists $m$ such that for every distribution $q$ over $\cX\times[-R,R]$, with probability at least $1-\delta$ over $S\sim q^n$ of size $n\ge m$,
\[
\EEs{(x,y)\sim q}{\bigl|\cA_S(x)-y\bigr|} \le \inf_{f\in\cF} \EEs{(x,y)\sim q}{|f(x)-y|} + \gamma + \varepsilon,
\]
where $\cA_S$ denotes the hypothesis output by $\cA$ on input sample $S$.
Finally, the $\gamma$-fat-shattering dimension of $\cF$, denoted $\pd_\gamma(\cF)$, is the largest integer $d$ for which there exist points $x_1,\dots,x_d\in\cX$ and thresholds $r_1,\dots,r_d\in\mathbb{R}$ such that for every $b\in\{\pm 1\}^d$ there exists $f\in\cF$ satisfying
\[
f(x_i)\ge r_i+\tfrac{\gamma}{2}\ \text{ if } b_i=+1,
\quad\text{and}\quad
f(x_i)\le r_i-\tfrac{\gamma}{2}\ \text{ if } b_i=-1.
\]
If for arbitrarily large $d$ there exist points $x_1,\dots,x_d$ and thresholds $r_1,\dots,r_d$ satisfying the above, we say that $\pd_\gamma(\cF)=\infty$.

When~$\cF$ consists of $\{0,1\}$-valued functions and the scale $\gamma<1$, the fat-shattering dimension coincides with the VC dimension, and these notions reduce to their classical binary counterparts. Thus, the real-valued framework can be viewed as a strict generalization of the binary setting. 
%\textcolor{red}{The following result combine some known results with our own new ones to give a real-valued version of the fundamental theorem of PAC learning.}

\begin{mdframed}[backgroundcolor=theoremcolor,linewidth=0pt]
    \begin{restatable}{theorem}{UCintro}\label{thm:UCintro}
Let $\cF\subset [-R,R]^\cX$ be a bounded function class and let $\gamma>0$. Then, the following are equivalent:
\begin{enumerate}
    \item $\cF$ is $\frac{\gamma}{2}$-learnable;
    \item $\cF$ satisfies $\gamma$-uniform convergence;
    \item for every $\gamma' > \gamma$, the $\gamma'$-fat-shattering dimension of $\cF$ is finite.
\end{enumerate}
\end{restatable}
\end{mdframed}
The proof appears in \Cref{proof:UCintro}. The implications $(1)\Rightarrow(3)$ and $(2)\Rightarrow(3)$ where established by ~\cite{BARTLETT1998}. The main new implication is $(3)\Rightarrow(2)$: finite fat-shattering dimension at the corresponding scale yields uniform convergence. Since uniform convergence implies agnostic learnability by the ERM principle, this also recovers Long’s learnability implication~\cite{Long01} ($(3)\Rightarrow(1)$), with sharper sample-complexity bounds matching the optimal dependence on the excess error parameter up to logarithmic factors. Removing these logarithmic factors remains open.

\paragraph{Covering Numbers.}
The derivation of uniform convergence from finite combinatorial dimension follows the classical approach of Vapnik and Chervonenkis. The argument proceeds via symmetrization (also known as double sampling), which reduces the problem to controlling the behavior of the class on a finite sample. In the binary case, the behavior of the class is controlled using the Sauer-Shelah Lemma~\cite{Sauer1972}, which bounds the number of distinct labelings induced by the class on a sample of size $n$. This combinatorial bound allows one to replace an infinite union bound with a finite one.

In the real-valued setting, the role of the Sauer--Shelah Lemma is played by covering numbers. Concretely, for $\cF \subset [-R,R]^{\cX}$, the empirical $\ell_\infty$ covering number $N(\cF,\gamma,n)$ is the smallest integer $N$ for which for any sequence $x_1,\dots,x_n$, there exist functions $g_1,\dots,g_N:\cX \to \R$ such that for every~$f \in \cF$, there exists $i \in [N]$ satisfying
\(
\max_{j \in [n]} |f(x_j) - g_i(x_j)| \le \gamma.
\)
The quantity
\[
H(\cF,\gamma,n) := \log N(\cF,\gamma,n)
\]is called the (empirical) metric entropy, and serves as the appropriate analogue of the combinatorial growth function. A classical result of Dudley et al.~\cite{dudley1991uniform} shows that sublinear entropy growth, namely $H(\cF,\gamma/2,n)=o(n)$, implies $\gamma$-uniform convergence.
% This was improved by \cite{BARTLETT1998}, where it was shown that one can replace the $\frac{\gamma}{8}$ scale with $\frac{\gamma}{2}$.

A long line of work has sought to bound this entropy in terms of the fat-shattering dimension. Alon et al.~\cite{Alon1997-scale-sensitive} showed that if $\pd_\gamma(\cF)<\infty$, then 
\[
H(\cF,2\gamma,n) = O(\log^2 n)
\]
 In addition to establishing this bound, they explicitly asked whether the $\log^2 n$ dependence is inherent, or whether it can be improved to $O(\log n)$. 
 % This was later improved by Bartlett and Long~\cite{BARTLETT1998}, who showed that the same $O(\log^2 n)$, but replacing the $2\gamma$ with $c\gamma$  for any constant $c>2$, but without improving the logarithmic exponent.
Subsequent work of Rudelson and Vershynin~\cite{rudelson2006combinatorics} showed that if one allows the scale to increase further, then the entropy bound can be improved: for every $\varepsilon>0$,
\[
H\!\left(\cF,\tfrac{c\gamma}{\varepsilon},n\right) = O\bigl(\log^{1+\varepsilon} n\bigr).
\]
They further asked whether the exponent of the logarithm can be reduced all the way to $1$, i.e., whether an $O(\log n)$ bound is achievable.

These results are obtained by first controlling packing numbers and then passing to covering numbers, which necessarily incurs a loss in the scale. 
Our approach departs from this line of work by bounding covering numbers directly, without passing through packing arguments. This allows us to obtain sharp bounds at the optimal scale.
Our second main result establishes the optimal scale at which the metric entropy becomes sublinear, and refines the known bounds on its growth across different scales.

\begin{mdframed}[backgroundcolor=theoremcolor,linewidth=0pt]
\begin{restatable}{theorem}{coverDichIntro}\label{thm:cover-Dich-intro}
Let $\cF \subset [-R,R]^{\cX}$ be a bounded function class, and let 
\[\gamma^\star= \inf\{\gamma>0 : \pd_\gamma(\cF)<\infty\}.\]
Then\footnote{Throughout, the asymptotic notation treats $n$ as the parameter tending to infinity; the implicit constants may depend on $\gamma$ and on the fat-shattering dimension at scale $\gamma$.},
\[
H(\cF,\gamma,n)
=
\begin{cases}
\Theta(n) & \gamma \in (0,\frac{\gamma^\star}{2}), \\[4pt]
O(\log^2 n) & \gamma\in (\frac{\gamma^\star}{2}, 2\gamma^\star], \\[4pt]
O(\log n) & \gamma\in (2\gamma^\star,R).
\end{cases}
\]
Moreover, the second regime is tight for \(\gamma\in\left(\frac{\gamma^\star}{2},\gamma^\star\right)\): there exist \(\cF\) s.t.\ for all \(\gamma<\gamma^\star\), 
\[
H(\cF,\gamma,n)=\Omega\!\left(\log^{2-o(1)}n\right).
\] 
The third regime is tight: there exist \(\cF\) such that \(H(\cF,\gamma,n)=\Omega(\log n)\) for \(\gamma\in (2\gamma^\star,R)\).
\end{restatable}
\end{mdframed}

The proof appears in \Cref{proof:coverDichIntro}. 
The same result holds for any bounded function class $\cF \subseteq \mathbb{R}^{\cX}$; we restrict to the range $[0,1]$ for notational convenience.

% This theorem shows that the classical $O(\log^2 n)$ bound of \cite{Alon1997-scale-sensitive} already holds at the optimal scale above $\gamma^\star/2$, and that an $O(\log n)$ bound is achievable at scale $2\gamma^\star$. Moreover, our result shows that the $O(\log^2 n)$ bound is essentially optimal throughout the range $\gamma \in (\gamma^\star/2,\gamma^\star]$, i.e., up to a factor of two from the optimal scale, while a logarithmic bound becomes possible already at scale~$2\gamma^\star$. 

In light of the $O(\log^{1+\varepsilon} n)$ bounds of Rudelson and Vershynin~\cite{rudelson2006combinatorics} mentioned above, which left open whether the exponent can be improved to $1$, our result resolves this question. It also answers the question of Alon et al.~\cite{Alon1997-scale-sensitive} on whether the $\log^2 n$ bound can be improved to $\log n$, by showing that such an improvement is possible at scales above $2\gamma^\star$, but impossible (up to lower-order terms) below $\gamma^\star$.

The first and third regimes are tight and well understood from prior work: below $\gamma^\star/2$, linear entropy growth holds for every class, while above $2\gamma^\star$ there exist classes with entropy $\Omega(\log n)$. Thus, on the lower bound side our main contribution is that the $O(\log^2 n)$ bound is sometimes tight (up to lower-order terms).
We remark that the case $\gamma = \gamma^\star/2$ is not covered by the theorem, and there are simple examples exhibiting either behavior; see \Cref{prop:crit-val-cover}.

We leave open the problem of characterizing the precise behavior of the entropy in the intermediate range $\gamma^\star < \gamma \le 2\gamma^\star$, as well as obtaining sharp sample complexity bounds for learning and uniform convergence at the optimal scale. We also leave open the problem of obtaining optimal relationships between the fat-shattering dimension and covering numbers under other $\ell_p$ metrics. Such covering numbers have been studied extensively, e.g.~\cite{BartlettKulkarniPosner1997,Anthony_Bartlett_1999,Mendelson2002,MendelsonVershynin2003}.

In \Cref{sec:closure}, we illustrate a simple application of our results by deriving closure properties and composition theorems for uniform convergence and learnability. We next turn to a different - and perhaps more surprising - application, where we use the sharp scale-sensitive results to resolve an open question of \cite{aiyer2026theoretical} concerning the evaluability and estimability of generative models. Remarkably, the \emph {quantitative} improvement from a factor of $2$ to $1$ in the scale unlocks a \emph {qualitative} characterization: a sharp dichotomy between estimability and evaluability. We elaborate on this connection in the following section.

\subsection{Evaluability and Estimability of Generative Models}

How can one test whether a generative model truly generalizes?  
Consider a model trained to generate images or short musical tunes from a dataset. When it produces new samples, are these genuinely novel, or merely recombinations - or even copies - of the training data? Given only sample access to the true distribution, can we reliably determine whether the generator generalizes well? Can indistinguishability between generated and real samples be tested from finite data?

In practice, such questions are addressed via evaluation metrics that quantify the discrepancy between the generated and target distributions. A recent work by \cite{aiyer2026theoretical} introduced a general framework for studying such metrics as tools for generalization testing. A central class of metrics in this framework is given by integral probability metrics (IPMs):
\[
d_{\cF}(q, q^\star) := \sup_{f \in \cF} \left| \E_q[f] - \E_{q^\star}[f] \right|,
\]
where $\cF$ is a class of real-valued functions $f:\cX \to \R$ defined on the data domain $\cX$. Intuitively, $d_{\cF}$ measures how well one can distinguish $q$ from $q^\star$ using tests from $\cF$. 

Classical examples of IPMs include total variation distance, obtained when $\cF$ is the set of all measurable functions bounded in $[0,1]$, and the Wasserstein-1 distance, which arises when $\cF$ is the class of $1$-Lipschitz functions with respect to a given metric on $\cX$.

We next formalize the setup and introduce evaluability and estimability, following the framework of \cite{aiyer2026theoretical}.

Let $\cX$ denote the data domain. A \emph{model} is a probability distribution over $\cX$, representing a generative model of the data. We let $\cM = \Delta(\cX)$ denote the set of all probability distributions over $\cX$. We assume there exists an unknown ground-truth model $q^\star \in \cM$ from which data is sampled.

% A \emph{score function} is a map $s: \cM \times \cX^* \to \R$ which, given a model $q$ and evaluation data $\sev$ sampled from $q^\star$, outputs a real-valued score $s(q,\sev)$. Intuitively, the score function is an algorithmic proxy for the true distance $d_{\cF}(q,q^\star)$: since the latter depends on the unknown distribution $q^\star$, the score function uses samples from $q^\star$ to approximate it and to compare competing models.

An evaluation algorithm $\cA$ takes as input two candidate models $q_1,q_2 \in \cM$ and an evaluation sample $S_{\text{eval}}$ drawn from $q^\star$, and outputs one of the two models, interpreted as the model it judges to be better. A common class of evaluation algorithms operates by assigning each model a score based on the evaluation data and selecting the model with the smaller score.

\paragraph{Evaluability.}
For $c \geq 1$, we say that $d_{\cF}$ is \emph{$c$-weakly evaluable} if there exists a sample complexity function $m_{\text{evl}}:(0,1)^2 \to \NN$ and an evaluation algorithm $\cA$ such that for every pair of models $q_1,q_2 \in \cM$,
% score function $s:\cM \times \cX^* \to \R$ 
% for which the following holds: 
every $\epsilon,\delta \in (0,1)$ and every ground-truth model $q^\star$, given IID evaluation data $S_{\text{eval}} = \{x_1,\ldots,x_m\}$ of size $m \ge m_{\text{evl}}(\epsilon,\delta)$ with $x_i \sim q^\star$, with probability at least $1-\delta$,
\[
\cA(\{q_1,q_2\},S_{\text{eval}})=q_1 \;\;\Longrightarrow\;\;
d_{\cF}(q_1,q^\star) \le c \cdot d_{\cF}(q_2,q^\star) + \epsilon\,,
\]
and
\[
\cA(\{q_1,q_2\},S_{\text{eval}})=q_2  \;\;\Longrightarrow\;\;
d_{\cF}(q_2,q^\star) \le c \cdot d_{\cF}(q_1,q^\star) + \epsilon\,.
\]
Thus, evaluability asks whether two candidate models can be reliably compared using only samples from the ground-truth distribution. The evaluation algorithm need not estimate the true distances themselves; it only needs to select a model whose distance to $q^\star$ is no worse than that of the alternative, up to a multiplicative factor $c$ and an additive error $\epsilon$. 

We say that $d_{\cF}$ is \emph{weakly evaluable with optimal scale $c$} if it is $c$-weakly evaluable and there does not exist any $c' < c$ such that it is $c'$-weakly evaluable.
When $c=1$, we say that $d_{\cF}$ is \emph{strongly evaluable}.

\paragraph{Estimability.}
We say that $d_{\cF}$ is \emph{estimable} if there exists a sample complexity function $m_{\text{est}}:(0,1)^2 \to \NN$ and a score function $s$ such that for any $\epsilon,\delta \in (0,1)$, any ground-truth model $q^\star$, and any model $q \in \cM$, given IID evaluation data $S_{\text{eval}} = \{x_1,\ldots,x_m\}$ of size $m \ge m_{\text{est}}(\epsilon,\delta)$ with $x_i \sim q^\star$, with probability at least $1-\delta$,
\[
\big| s(q,S_{\text{eval}}) - d_{\cF}(q,q^\star) \big| \le \epsilon.
\]
In contrast to evaluability, which concerns relative comparison, estimability requires accurately approximating the \emph{value} of the metric itself. In particular, estimability implies strong evaluability.

A central result of \cite{aiyer2026theoretical} establishes a dichotomy for IPMs induced by $\{0,1\}$-valued test functions: such metrics are either estimable (and hence strongly evaluable), or weakly evaluable with optimal scale $c=3$. Their proof, however, is tailored to the $\{0,1\}$-valued setting, relying on tools such as VC-dimension and uniform convergence for binary-valued classes.

Extending this result to general real-valued test classes is significantly more challenging. In particular, the binary setting allows one to leverage VC-dimension arguments directly, whereas in the real-valued case the appropriate complexity measure is the fat-shattering dimension, and its connection to uniform convergence is more delicate. A key obstacle is that the classical results of \citep{BARTLETT1998} only guarantee uniform convergence at scale $2\gamma$, which is insufficient for the evaluability guarantees required here.

Our approach overcomes this difficulty by establishing uniform convergence at the optimal scale $\gamma$. Crucially, eliminating this factor-$2$ loss is essential: without it, the argument breaks down and yields no non-trivial weak evaluability guarantee for real-valued IPMs.
\begin{mdframed}[backgroundcolor=theoremcolor,linewidth=0pt]
\begin{restatable}{theorem}{realeval}\label{thm:realeval}
Every bounded IPM is either:
\begin{enumerate}
    \item estimable, and hence strongly evaluable; or
    \item weakly evaluable with optimal scale $c=3$.
\end{enumerate}
\end{restatable}
\end{mdframed}

We further characterize this dichotomy in terms of the fat-shattering dimension of the test class.

\begin{restatable}{prop}{fateval}\label{prop:fateval}\label{prop:fateval}
Let $\cF$ be a class of bounded real-valued functions.
\begin{enumerate}
    \item If $\pd_{\gamma}(\cF) < \infty$ for all $\gamma > 0$, then $d_{\cF}$ is estimable (and hence strongly evaluable).
    \item If there exists $\gamma \in (0,1/2)$ such that $\pd_\gamma(\cF)=\infty$, then $d_{\cF}$ is weakly evaluable with optimal scale $c=3$.
\end{enumerate}
\end{restatable}
The proof is in \cref{proof:fateval}.
The above results are qualitative in nature and are most meaningful over infinite domains. Indeed, over a finite domain $\cX$ of size $N$, every bounded IPM is estimable with $O(N/\epsilon^2)$  samples, and hence strongly evaluable. 

This raises the question of whether quantitative analogues of the above dichotomy hold in the finite-domain setting. For example, as a function of the domain size $N$, how does the sample complexity of weak evaluability behave as the approximation factor improves from $c=3$ to $c=1$? Moreover, is the sample complexity required for strong evaluability comparable to that required for estimability?

Understanding these questions would provide a more refined picture of the trade-offs between sample complexity and evaluation in practical settings.

\subsection{Composition and Closure Properties}\label{sec:closure}
We now return to the setting of real-valued function learning and uniform convergence, and describe a more direct and standard application of \Cref{thm:UCintro}, in contrast to the more involved application to evaluability discussed above.

Characterizations of learnability using combinatorial dimensions are ubiquitous in learning theory, and a natural litmus test for their usefulness is whether they can be used to derive new results that are otherwise harder to prove. While the usefulness of the fat-shattering dimension is well established, we use it here to illustrate another application of \Cref{thm:UCintro}. 

Specifically, we show that the properties of being $\gamma$-learnable and of satisfying $\gamma$-uniform convergence are invariant under natural operations. By \Cref{thm:UCintro}, this reduces to showing that the property of having finite $\gamma$-fat-shattering dimension is invariant under these operations. The latter is essentially well known and relies on arguments from prior work, which we include here for completeness. 
The novel ingredient is therefore \Cref{thm:UCintro}, which establishes the equivalence between finite fat-shattering dimension and learnability and uniform convergence at the optimal scale, thereby yielding closure properties for these notions.

As a first example, consider the dual class. For a function class $\cF \subset [0,1]^{\cX}$, its dual class $\cF^\star$ is a class of functions over $\cF$, defined by
\[
\cF^\star = \{\delta_x : x \in \cX\},
\]
where each $\delta_x : \cF \to [0,1]$ is given by $\delta_x(f) = f(x)$. Assouad’s bound \cite{assouad1983dual} relates the VC dimension of a class to that of its dual. This result extends naturally to the setting of $\gamma$-shattering when a common threshold $r \in (0,1)$ is used across all points. A simple application of the pigeonhole principle then shows that $\pd_{\gamma}(\mathcal{F})$ is finite if and only if $\pd_{\gamma+\varepsilon}(\mathcal{F}^\star)$ is finite for every $\varepsilon > 0$. As a corollary, we obtain the following result, which shows that $\gamma$-uniform convergence is preserved under taking duals.

\begin{restatable}{prop}{DualUC}\label{prop:DualUC}
Let $\gamma > 0$, and let $\cF$ be a bounded function class that satisfies $\gamma$-uniform convergence. Then the dual class $\cF^\star$ also satisfies $\gamma$-uniform convergence.
\end{restatable}
The proof is in \cref{proof:DualUC}.

Another closure property is stability under aggregation. Let $G : [0,1]^k \to [0,1]$ be a function, and let $f_1,\dots,f_k : \cX \to [0,1]$. Their aggregation with respect to $G$ is defined by
\[
G(f_1,\dots,f_k)(x) = G\big(f_1(x),\dots,f_k(x)\big).
\]
Given a sequence of function classes $\cF_1,\dots,\cF_k$, we define the aggregated class
\[
G(\cF_1,\dots,\cF_k) = \{\, G(f_1,\dots,f_k) : f_i \in \cF_i \text{ for all } i \,\}.
\]

The fat-shattering dimension of such aggregated classes was studied in \cite{Idan2024-k-fold}. Our next result is similar in spirit to Theorem~2 of \cite{Idan2024-k-fold}, but avoids the blow-up in the fat-shattering scale. As a consequence, we obtain that uniform convergence is closed under aggregation.

\begin{restatable}{prop}{Aggregation}\label{prop:Aggregation}
Let $\cF_1,\dots,\cF_k\subset [0,1]^\cX$ be bounded function classes, each satisfying $\gamma$-uniform convergence, and let $G : [0,1]^k \to [0,1]$ be a $1$-Lipschitz function with respect to the $\ell_\infty$ metric. Then the aggregated class $G(\cF_1,\dots,\cF_k)$ also satisfies $\gamma$-uniform convergence.
\end{restatable}
The proof is in \cref{sec:compandclosure}, and follows as a consequence of Theorem \ref{thm:UCintro}, by showing that an aggregation of covers is a cover for the aggregated class.

% \subsection{Organization}
% Here we detail the organization of the rest of the paper, primarily the appendix. \cref{sec:overviews} provides in-depth overviews of the arguments within the three main theorems. We discuss several related works organically throughout the main paper. For a dedicated related works section, we refer the reader to \cref{sec:work}. In \cref{sec:convergencethms} we prove \cref{thm:UCintro}, \cref{thm:cover-Dich-intro}, \cref{prop:DualUC}, and \cref{prop:Aggregation}, along with supporting lemmas. Finally, \cref{sec:evaluabilitythm} contains the proofs for \cref{thm:realeval} and \cref{prop:fateval}, along with supporting lemmas. 

\section{Related Work}
The relationship between fat-shattering dimension, uniform convergence, and covering numbers is well established in statistical learning theory. The characterization of uniform convergence via covering numbers is due to Dudley \cite{dudley1978central,Dudley1984ACO,dudley1991uniform}. The fat-shattering dimension was introduced by Kearns and Schapire \cite{KEARNS1994} as a complexity measure capturing the sample complexity of approximate learnability. In a seminal work, Alon et al. \cite{Alon1997-scale-sensitive} showed that finite $\gamma$-fat-shattering dimension implies a $O(\log^2 n)$ bound on the metric entropy at scale $2\gamma$, thereby establishing a connection between fat-shattering and uniform convergence. They also conjectured that the $\log^2 n$ dependence could be improved to $\log n$.

Building on these ideas, Bartlett and Long \cite{BARTLETT1998} made significant progress toward understanding the optimal dependence on the scale. In particular, they improved the result of \cite{Alon1997-scale-sensitive} by showing that finite $\gamma$-fat-shattering dimension implies a $O(\log^2 n)$ bound on the metric entropy at scale $\gamma+\varepsilon$. They also established complementary lower bounds which, combined with classical results of Dudley and Pollard \cite{dudley1991uniform,Pollard1984ConvergenceOS}, yield an equivalence between $\gamma$-fat-shattering and uniform convergence up to a factor of $2$ in the scale. {In subsequent work, Long \cite{Long01} gave improvements on the scale at which one can get sub-liner covering bounds and gave a tight equivalence between finite fat-shattering dimension and learnability.} A comprehensive overview of those techniques appears in \cite{Anthony_Bartlett_1999}.

A complementary line of work focuses on improving covering number bounds themselves, rather than the scale at which they are obtained. These works often consider $\ell_1$ or $\ell_2$ metrics instead of $\ell_\infty$, leading to improved quantitative bounds on sample complexity. Early contributions include \cite{BartlettKulkarniPosner1997,CESABIANCHI1998}, followed by a sequence of works by Mendelson and Vershynin \cite{mendelson2002entropy,MendelsonVershynin2003,Mendelson2002}, which developed refined entropy bounds. This line of research culminated in the work of Rudelson and Vershynin \cite{rudelson2006combinatorics}, who obtained strong covering number bounds under various norms. Notably, they made progress toward the conjecture of \cite{Alon1997-scale-sensitive} by improving the $\log^2 n$ dependence to $\log^{1+\varepsilon} n$, at the cost of a multiplicative factor of $1/\varepsilon$ in the scale. They further conjectured that this dependence could be improved to $O(\log n)$.

There are several additional directions related to these questions. Colomboni et al. \cite{COLOMBONI2025-improved-UC} used refined covering bounds to derive sharp sample complexity guarantees for uniform convergence. Hanneke et al. \cite{hanneke2019sample} and Attias and Kontorovich \cite{attias2024agnostic} established connections between fat-shattering dimension and approximate sample compression schemes. Finally, recent works such as \cite{Idan2024-k-fold,KLEER2023-Dual} study structural properties of fat-shattering dimension, including closure under aggregation and duality.

The framework of evaluability was recently introduced by \cite{aiyer2026theoretical}, 
which provides a theoretical formulation of the problem of statistically evaluating the performance of generative models from finite samples. 
This line of work is motivated in part by limitations of commonly used empirical scores, such as perplexity and its variants 
\citep{fang2025wrongperplexitylongcontextlanguage,hu2024perplexityreflectlargelanguage,pillutla2021mauvemeasuringgapneural,martins2020sparsetextgeneration}. 
Commonly used metrics include KL divergence, total variation distance, and other divergences.

One general class of metrics is integral probability metrics (IPMs), introduced by \cite{MullerIPM}, 
which subsume important examples such as total variation distance and the Wasserstein-1 distance. 
IPMs have been extensively studied in the context of density estimation, where, given sample access to a ground-truth distribution and a class of candidate distributions, the goal is to output the closest distribution in the class. 
In this setting, \cite{yatracos} established an upper bound with multiplicative factor $3$, and \cite{bousquet2019optimal} showed that this factor is optimal.  \cite{aiyer2026theoretical} extended this factor of $3$ from density estimation to evaluability for binary $\cF$. 
In this work, we further generalize these results to real-valued function classes.

\section{Proof Overviews}\label{sec:overviews}

\subsection{Proof Overview of \cref{thm:UCintro}}
The fact that a finite fat-shattering dimension  is a necessary condition for uniform convergence and learnability is well established, and follows by considering the uniform distribution over large shattered sets; see, Theorem 22 and Theorem 26 in \cite{BARTLETT1998}. The main challenge is to show the converse direction, namely that finite fat-shattering dimension implies uniform convergence and learnability.

The standard approach for deriving uniform convergence bounds, as in \cite{BARTLETT1998,Alon1997-scale-sensitive,COLOMBONI2025-improved-UC}, proceeds via symmetrization combined with classical probabilistic arguments, reducing the problem to bounding covering numbers. Concretely, one needs to show that for all $\varepsilon > 0$,
\[
\lim_{n\to\infty}\frac{H(\mathcal{F},\tfrac{\gamma}{2}+\varepsilon,n)}{n} = 0,
\]
where we recall $H(\cF,\gamma,n)=\log N(\mathcal{F},\gamma,n)$ denotes the empirical metric entropy  of $\mathcal{F}$ with respect to the $\ell_\infty$ metric.

In prior works, the covering number is typically upper bounded via the packing number, which is easier to control. This reduction allows one to assume the existence of a well-separated set and bound its size using Sauer-type lemmas or volume arguments. However, this approach has an inherent limitation: a finite $\gamma$-fat-shattering dimension only guarantees a bound on the $\gamma$-packing number. This, in turn, yields a bound on the $\gamma$-covering number, leading to $2\gamma$-uniform convergence {and $\gamma$-learnability}. To improve upon the result of \cite{BARTLETT1998}, one must therefore bound the covering number directly, without relying on packing arguments.

This introduces the main technical difficulty: \emph{bounding the covering number directly requires constructing an explicit cover, which may not consist of functions from the class itself}. {To address this, we take a similar approach to \cite{Long01}, and employ the theory of partial concept classes \cite{Long01,alon2022theory}.}

Let $R > 0$, and let $\mathcal{F} \subset [-\tfrac{R}{2}, \tfrac{R}{2}]^n$ be a bounded function class\footnote{It suffices to assume that $\cF$ is inside an $\ell_\infty$ ball of radius $\frac{{R}}{2}$. The assumption it is centered at $0$ is for notational convenience} with finite $\gamma$-fat-shattering dimension $\pd_\gamma(\cF)=d$. We associate to each $f \in \mathcal{F}$ a partial binary function $h_f$ defined by
\[
h_f(x) =
\begin{cases}
1 & f(x) \ge \tfrac{\gamma}{2},\\
-1 &  f(x) \le -\tfrac{\gamma}{2},\\
\text{undefined} & \text{otherwise}.
\end{cases}
\]
The VC dimension of the induced partial function class will be bounded by $d$. Hence, by \cite{alon2022theory}, it admits a disambiguation of sub-exponential size: that is, there exists a set $H \subset \{-1,1\}^n$ with $|H| = n^{O(d{\log n})}$ such that for every $f \in \mathcal{F}$, there exists $h \in H$ satisfying $h(x) = h_f(x)$ whenever $h_f(x)$ is defined. 

This implies that if $h(x) = 1$, then $f(x) \in [-\tfrac{\gamma}{2}, \tfrac{R}{2}]$, and if $h(x) = -1$, then $f(x) \in [-\tfrac{R}{2}, \tfrac{\gamma}{2}]$. Consequently, by associating, with each $h\in H$, a function $\phi_h$, such that $\phi_h(x)= \tfrac{R-\gamma}{4}$ when $h(x)=1$ and $\phi_h(x)= \tfrac{-R+\gamma}{4}$ when $h(x)=-1$, we obtain a cover of $\mathcal{F}$ by $n^{O(d{\log n})}$ boxes of radius $\tfrac{\gamma + R}{4}$. %$\mathcal{F}$ can be covered by $n^{O(d\textcolor{red}{\log n})}$ boxes of diameter $\tfrac{\gamma + r}{2}$, i.e., of radius $\tfrac{\gamma + r}{4}$.

To obtain finer scales, we iterate this construction. Starting from a coarse cover and recursively refining each box, the radius decreases geometrically and converges to $\tfrac{\gamma}{2}$, while at each step we incur a multiplicative increase of $n^{O(d{\log n})}$ in the covering number size. Thus after $k$ iterations we have a bound of $n^{O(dk{\log n})}$ on the covering number size at scale $\tfrac\gamma2+\frac{1}{2^k}$,
yielding an effective bound on the covering number at all scales strictly larger than $\tfrac{\gamma}{2}$ 
\[
H(\cF,\tfrac\gamma2+\varepsilon,n)=O(d\log\frac{1}{\varepsilon}\log^2 n).
\]

The $\frac{\gamma}{2}$-learnability result follows from a covering argument based on the algorithm of Theorem 21 in \cite{BARTLETT1998}, combined with our improved covering bound to obtain a sharp rate. Given a sample ${(x_i, y_i)}_{i=1}^{2n}$ and a test point $x$, the algorithm constructs a $(\frac{\gamma}{2} + \varepsilon)$-cover of the restriction of $\mathcal{F}$ to these $2n+1$ points. It then outputs $f(x)$, where $f \in \mathcal{F}$ minimizes the empirical loss over the first $n$ points. Standard concentration inequalities ensure that the excess loss is controlled by the covering scale, provided the size of the cover is appropriately bounded.

% \subsection{Proof Overview of \cref{thm:cover-Dich-intro}}

% The first part is straightforward: if a set of size $n$ is $\gamma$-shattered, then for any $\gamma'<\frac{\gamma}{2}$, we have that any $\gamma'$-cover of its projection must contain at least $2^n$ functions. This implies that $H(\cF,\gamma',n) = \Omega(n)$, as observed in \cite{Alon1997-scale-sensitive,BARTLETT1998}. By volume estimates we can always find a $\gamma$-cover using $\frac{1}{\gamma^n}$ functions, so we indeed get that for $\gamma\in (0,\frac{\gamma^\star}{2})$ we have 
% \[
% H(\F,\gamma,n)=\Theta(n).
% \]

% The $O(\log ^2 n)$ bounds follows from the same covering argument used to establish Theorem~\ref{thm:UCintro}, as discussed above.

% To show it is tight in the $\gamma\in (\frac{\gamma^\star}{2},\gamma^\star)$ regime we use a result of \cite{alon2022theory} which shows that there disambiguation bound is tight. That is, there exist a partial concept class $\cF$, with finite VC dimension but no disambiguation of size smaller than $n^{(1+o(1))\log n}$. We embed this class into the real-valued setting by assigning the value $\tfrac{1}{2}$ to undefined points. In this construction, any covering at scale $\gamma\in (\frac{\gamma^\star}{2},\gamma^\star)$, induces a disambiguation, which yields the stated lower bound.

\subsection{Proof Overview of \cref{thm:cover-Dich-intro}}
We begin with the first regime. Suppose a set of size $n$ is $\gamma$-shattered. Then, for any $\gamma' < \gamma/2$, every $\gamma'$-cover of its projection must contain at least $2^n$ functions. It follows that $H(\cF,\gamma',n) = \Omega(n)$, as observed in \cite{Alon1997-scale-sensitive,BARTLETT1998}. On the other hand, standard volume estimates yield the existence of a $\gamma$-cover of size at most $(1/\gamma)^n$. Hence, for $\gamma \in (0,\gamma^\star/2)$,
\[
H(\F,\gamma,n) = \Theta(n).
\]

The $O(\log^2 n)$ upper bound in the intermediate regime follows from the same covering argument used in the proof of Theorem~\ref{thm:UCintro}.

To establish tightness for $\gamma \in (\gamma^\star/2,\gamma^\star)$, we appeal to a result of \cite{alon2022theory}, which shows that the disambiguation bound is tight. Specifically, they construct a partial concept class $\cF$ with finite VC dimension for which any disambiguation must have size at least $n^{(1-o(1))\log n}$. We embed this class into the real-valued setting by mapping $1$-labels to $\gamma^\star$, $-1$-labels to $-\gamma^\star$, and undefined points to $0$. 

In this construction, any $\gamma$-cover with $\gamma \in (\gamma^\star/2,\gamma^\star)$ induces a disambiguation as follows: given a $\gamma$-cover $H$, define $\Phi_h(x) = \sign(h(x))$. If $|f(x) - h(x)| < \gamma^\star$ for all $x$, then $f$ and $h$ agree in sign on all points where $f(x) \neq 0$. Thus $\Phi_h$ disambiguates the original partial class, yielding the desired lower bound.

Finally, we consider the regime $\gamma \in (2\gamma^\star, 1)$, where we rely on sample compression schemes. We begin by briefly recalling the notion of sample compression and how it leads to covering bounds.

Formally, we say that $\cF$ admits a $\gamma$-sample compression scheme of size $k$ if there exists a reconstruction function  
\[
\rho : (\cX \times \R)^k \to \R^{\cX}
\]
such that for every finite set $X \subseteq \cX$ and every function $f \in \cF$, there exists a sequence  
\[
T = \{(x_i, f(x_i))\}_{i=1}^k, \quad x_i \in X,
\]
for which the reconstructed function $\rho(T)$ satisfies  
\[
|\rho(T)(x) - f(x)| \le \gamma \quad \text{for all } x \in X.
\]

Such a compression scheme yields a cover of polynomial size when $k = O(1)$ (i.e., independent of the size of $X$), since discretizing the function values at scale $\varepsilon$ and enumerating all possible reconstructions from $k$ labeled points yields at most $(n/\varepsilon)^k$ candidate functions. This gives a cover of size polynomial in $n$, and hence a logarithmic bound on the metric entropy.

It therefore remains to construct such a compression scheme.
We follow the approach of \cite{moran2016sample}, who studied sample compression in the binary-labeled setting, and adapt their argument to the real-valued case. Fix an empirical risk minimization (ERM) rule that maps any realizable dataset $T = \{(x_i, y_i)\}$ (that is, a dataset for which there exists $f \in \cF$ with $y_i = f(x_i)$ for all $i$) to a hypothesis $h_T \in \cF$ that is consistent with the dataset.

By uniform convergence, if we take $T$ to be of size $m = m(\varepsilon, 1/3)$, where $m(\varepsilon, \delta)$ is the uniform convergence sample complexity (and $1/3$ can be replaced by any constant $\delta < 1$), then for any finite domain $X = \{x_1,\dots,x_n\}$, any distribution $p$ over $X$, and any function $f \in \cF$, there exists a realizable dataset $T \subseteq X$ (with labels given by $f$) such that the corresponding hypothesis $h_T$ satisfies
\[
\mathbb{E}_{x \sim p}[|f(x) - h_T(x)|] \le \gamma + \varepsilon.
\]
This follows from the uniform convergence guarantee applied to a random draw of $T$ from $p$.

A minimax argument now implies that there exists a distribution over such datasets $T$ such that, when $T$ is drawn from this distribution, the function $h_T(x)$ is close to $f(x)$ in expectation for every point~$x \in X$.

We then use uniform convergence for the dual class to sparsify this distribution. Specifically, we approximate it by a uniform distribution over a finite multiset of datasets $T_1,\dots,T_{m^\star}$, where $m^\star=m^\star(\varepsilon,\frac{1}{3})$ is the uniform convergence sample complexity of the dual class at scale $\gamma$. This sparsification incurs an additional error of $\gamma + \varepsilon$.

Averaging the corresponding hypotheses $h_{T_i}$ then yields a function that is $(2\gamma + O(\varepsilon))$-close to $f$ on all of $X$, giving the desired compression scheme.

Crucially, this argument relies on the fact that the dual class satisfies $\gamma$-uniform convergence. This is not automatic, and follows from our \Cref{thm:UCintro} (see \Cref{prop:DualUC}). Without this property, the sparsification step would not be possible.

\subsection{Proof Overview of \Cref{thm:realeval}}

\Cref{thm:realeval} follows as a direct corollary of \Cref{prop:fateval}. Indeed, the proposition provides a complete characterization of the behavior of $d_{\cF}$ in terms of the fat-shattering dimension: if $\pd_{\gamma}(\cF)$ is finite for all $\gamma>0$ then $d_{\cF}$ is estimable, while if $\pd_{\gamma}(\cF)=\infty$ for some $\gamma\in(0,1/2)$ then $d_{\cF}$ is weakly evaluable with optimal scale $3$. Thus, we focus on the proof of \Cref{prop:fateval}.

The first item and the first part of the second item (that $d_{\cF}$ is $3$-weakly evaluable) follow from standard arguments presented in \cite{aiyer2026theoretical}. The main challenge is to show that the factor $3$ is optimal. The proof builds on the argument of \cite{bousquet2019optimal}, which applies to total variation distance - an example of an IPM induced by a class with infinite fat-shattering dimension. While their proof is technically involved, most of the required adaptations are straightforward. We therefore focus on the main conceptual step: reducing a general IPM with infinite fat-shattering dimension to the total variation setting.

The result of \cite{bousquet2019optimal} shows that on a domain of size $n$, any evaluation algorithm with approximation factor strictly smaller than $3$ requires at least $\sqrt{n}$ samples. Letting $n \to \infty$, this rules out $c$-evaluability for any $c < 3$ on infinite domains. Thus, it suffices to embed arbitrarily large total variation instances into $d_{\cF}$.

Let $\gamma^\star = \inf\{\gamma : \pd_\gamma(\cF) < \infty\}$, and fix $\gamma^- < \gamma^\star < \gamma^+$ arbitrarily close to $\gamma^\star$. Let $d = \pd_{\gamma^+}(\cF)$. Since $\pd_{\gamma^-}(\cF)=\infty$, we can find arbitrarily large $\gamma^-$-shattered sets. Let $X=\{x_1,\dots,x_N\}$ be such a set, and let $P$ be the uniform distribution over $X$.

We sample $n$ independent batches $X_1,\dots,X_n$, where each batch $X_i$ consists of $m=\tilde{O}(d/\varepsilon^2)$ i.i.d.\ samples from~$P$. With high probability, the following properties hold:
\begin{enumerate}
    \item The batches $X_1,\dots,X_n$ are pairwise disjoint (by taking $N$ large enough to avoid collisions).
    \item For every pair $i \ne j$, the distance between the corresponding empirical distributions satisfies
    \[
    d_{\cF}(X_i,X_j) \le \gamma^+ + \varepsilon.
    \]
    This follows from a uniform convergence argument at scale $\gamma^+$
\end{enumerate}
Let $P_i$ denote the uniform distribution over $X_i$, and let $\Delta$ be the set of all convex combinations of $\{P_1,\dots,P_n\}$.
For any mixtures $P_\alpha = \sum_i \alpha_i P_i$ and $P_\beta = \sum_i \beta_i P_i$, we show that
\[
d_{\cF}(P_\alpha,P_\beta) \approx \gamma^\star \cdot \mathrm{TV}(\alpha,\beta).
\]

The lower bound uses the $\gamma^-$-shattering property. Since the batches are disjoint subsets of the shattered set, for the subset of indices $S = \{i : \alpha_i > \beta_i\}$, there exists a function $f \in \cF$ that separates $\bigcup_{i \in S} X_i$ from its complement at scale $\gamma^-$. This function witnesses a discrepancy of
\[
\gamma^- \cdot \mathrm{TV}(\alpha,\beta).
\]

The upper bound uses the fact that, by uniform convergence, any two batch distributions $P_i$ and $P_j$ are close under $d_{\cF}$: namely, $d_{\cF}(P_i,P_j) \le \gamma^+ + \varepsilon$. A triangle-inequality argument then shows that moving mass between the $P_i$'s can increase $d_{\cF}$ by at most roughly $\gamma^+ + \varepsilon$ times the total amount of mass moved, giving
\[
d_{\cF}(P_\alpha,P_\beta)
\le
(\gamma^+ + \varepsilon)\mathrm{TV}(\alpha,\beta).
\]
Since $\gamma^-$ and $\gamma^+$ can be chosen arbitrarily close to $\gamma^\star$, and since $\varepsilon$ can be taken arbitrarily small, these two bounds match up to a factor arbitrarily close to $1$.

This reduction embeds a total variation problem of dimension $n$ into $d_{\cF}$, allowing us to invoke the lower bound argument of \cite{bousquet2019optimal}.

A crucial feature of the construction is that $\gamma^-$ and $\gamma^+$ can be taken arbitrarily close to $\gamma^\star$. This ensures that, for mixtures in $\Delta$, the comparison between $d_{\cF}$ and total variation has distortion arbitrarily close to $1$:
\[
\gamma^- \cdot \mathrm{TV}(\alpha,\beta)
\;\lesssim\;
d_{\cF}(P_\alpha,P_\beta)
\;\lesssim\;
\gamma^+ \cdot \mathrm{TV}(\alpha,\beta).
\]
The critical point is that the upper bound must not lose a factor $2$. While the lower bound already captures the correct scale, the difficulty lies in obtaining a sufficiently tight upper bound. In particular, the standard packing-based arguments used to relate uniform convergence to fat-shattering dimension - such as those underlying classical scale-sensitive bounds - incur an inherent factor-$2$ loss. Applied in our setting, such arguments would only yield
\[
d_{\cF}(P_\alpha,P_\beta)
\;\lesssim\;
2\gamma^+ \cdot \mathrm{TV}(\alpha,\beta),
\]
rather than the near-tight bound we require.
This loss is fatal for our purposes: it prevents obtaining the optimal factor $3$ and invalidates this proof strategy altogether.
% , and in fact rules out any separation factor strictly larger than $1$. \textcolor{red}{Shashaank: I think if we had factor-2 then we would have a weaker lower bound for $c<3/2$, right?}\hs{no, if i remember correctly, we need this factor of $2$ to be smaller than $3/2$ to obtain a weaker lower bound of $c>1$. the factor $2$ doesn't imply any meaningful lower bound.} Thus, a factor-$2$ degradation in the upper bound relating uniform convergence to fat-shattering invalidates this proof strategy altogether. 

In contrast, our use of the refined uniform convergence result at optimal scale avoids this loss and yields an upper bound matching the lower bound up to a factor arbitrarily close to $1$. This tight control is what allows the reduction to go through and ultimately enables the dichotomy.

\section{Proofs of Scale-Sensitive Learnability Results} \label{sec:convergencethms}
\subsection{Covering Lemmas}
\begin{lemma}\label{lem:covering-base}
    Let $\cF\subset [\frac{-r}{2},\frac{r}{2}]^\cX$ be a bounded set of functions, let $\gamma\in (0,r)$, and denote $\pd_\gamma(\cF)=d>0$. Then for any $n>1$ we have 
    \[
    N(\cF,\frac{r+\gamma}{4},n)\leq n^{7d\log n}
    \]
\end{lemma}
\begin{proof}
By restricting to an arbitrary set of size $n$ we may assume that $|\cX|=n$, and we wish to find a cover for the function class $\cF$. For each $f\in \F$ define the partial function $h_f$ by
\[
h_f(x)=
\begin{cases}
1 & \text{if } f(x)\ge \tfrac\gamma2,\\
-1 & \text{if } f(x)\le -\tfrac\gamma2,\\
\text{undefined} & \text{otherwise}.
\end{cases}
\]    
By the definition of the fat-shattering dimension, the VC dimension of the induced partial concept class
\[
\{h_f : f\in F\}
\]
is at most $d$. Hence, by Theorem~13 of \cite{alon2022theory}, there exists a set of functions $H\subset \{-1,1\}^\cX$ of size
\[
|H|\leq (n+1)^{(d+1)\log_2(n)+2}\leq n^{7d\log n},
\]
that disambiguates this partial concept class; that is, for every $f\in F$ there exists some $h\in H$ such that $h(x)=h_f(x)$ whenever $h_f(x)$ is defined.

For each $h\in H$, define $\Phi_h:\cX\to [-r,r]$ by \[
\Phi_h(x)=\begin{cases}
   \frac{r-\gamma}{4} \quad &h(x)=1,
   \\ \frac{\gamma-r}{4} \quad &h(x)=-1.
\end{cases}
\]
We claim that $\{\Phi_h\;: \; h\in H\}$ is the desired $\frac{r+\gamma}{4}$-cover. Indeed for any $f\in F$, if $h$ disambiguates $h_f$, then for any $x\in X$ with $h(x)=1$ we have $f(x)\in (\tfrac{-\gamma}{2},\tfrac r2]$, and for any $x\in X$ with $h(x)=-1$ we have $f(x)\in [\frac{-r}{2},\tfrac\gamma2)$. Thus we indeed get that for all $x\in \cX$ we have
\[
|f(x)-\Phi_h(x)|\leq\frac{r+\gamma}{4}.
\]
This gives the desired result.
\end{proof}

\begin{lemma}[Covering Number Bound]\label{lem:covering}
Let $\cF\subset[-R,R]^\cX $ be a bounded function class, let $\gamma>0$, and denote $\pd_\gamma(\cF)=d>0$. Then for any $\varepsilon>0$, and $n>1$ we have
\[
N(\cF,\tfrac\gamma2+\varepsilon,n)\le n^{14d\log\frac{2R}{\varepsilon}\log n},
\]
\end{lemma}
\begin{proof}
By restricting to an arbitrary set of size $n$ we may assume that $|\cX|=n$, and we wish to find a cover for the function class $\cF$. We prove by induction on $k$ that for every integer $k>0$,
\[
N\!\left(\cF,\tfrac\gamma2+\frac{2R}{2^{k+1}}\right)\le n^{7dk\log n} .
\]
Choosing $k$ such that $\frac{2R}{2^{k+1}}\le \varepsilon \le \frac{2R}{2^k}$ then yields the desired bound.

For the base case $k=0$, observe that $\cF$ can be covered by a single ball of radius $R$ in the $\ell_\infty$ metric.

For the inductive step, assume that $\cF$ can be covered by $n^{7d(k-1)\log n}$ balls of radius $\frac{\gamma}{2}+\frac{2R}{2^k}$. For each such ball $F$ with center $f$ we may apply \cref{lem:covering-base} to the translated class $F-f=\{g-f\;:\; g\in \cF\}$, and find a covering of $F$ of size $n^{7d\log n}$ with balls of radius $\frac{\gamma+(\gamma+\frac{4R}{2^k})}{4}=\frac{\gamma}{2}+\frac{2R}{2^{k+1}}$. Combining the bounds give a $(\frac{\gamma}{2}+\frac{2R}{2^{k+1}})$-covering of size \[
n^{7(k-1)d\log n}n^{7d\log n}=n^{7kd\log n}.
\]
From which we deduce the desired result.
\end{proof}

\subsection{Proof for \cref{thm:UCintro}}

The following is a standard refinement of an argument by \cite{dudley1991uniform}  relating uniform  convergence and  covering number. Similar results can be found in \cite{Anthony_Bartlett_1999,BARTLETT1998,Pollard1984ConvergenceOS}
\begin{lemma}[Uniform Convergence Rate vs Covering Number]
\label{lem:uniform-convergence-bound}
Let $\mathcal F\subseteq [-R,R]^{\mathcal X}$ be a bounded function class, and let
$\gamma>0$. Then for any distribution $q$ over $\mathcal X$ and any
$\delta,\varepsilon>0$, if
\[
n>\frac{128R^2}{\varepsilon^2}\left( H\left(\mathcal F,\frac{\gamma}{2}+\frac{\varepsilon}{8},n\right)+\log \frac{1}{\delta}\right),
\]
then
\[
\Pr_{S \sim q^n}\left[
\sup_{f \in \mathcal F}
\left|
\frac{1}{n}\sum_{x \in S} f(x) - \EEs{x \sim q}{f(x)}
\right|
> \gamma + \varepsilon
\right]
\leq
\delta .
\]
In particular, if $H(\cF,\tfrac\gamma2+\varepsilon,n)=o(n)$ for all $\varepsilon>0$, then $\cF$ satisfies $\gamma$-uniform convergence.
\end{lemma}
\begin{proof}
    Let $\varepsilon>0$, let $q$ be some distribution over $\cX$ and for every dataset $S=\{x_i\}_{i=1}^n$ define the random variable $Z_S$ by \[
    Z_S=\sup_{f\in \cF}|\EEs{x\sim q}{f(x)}-\frac{1}{n}\sum_{i=1}^n f(x_i)|.
    \]
    We need to show that $\Pr[Z_S>\gamma+\varepsilon]<\delta$. We will do that by showing \begin{align*}
        &\EEs{S\sim q^n}{Z_S}\leq \gamma+\frac{\varepsilon}{2},
        \\&Pr[Z_S>\E Z_S+\frac{\varepsilon}{2}]\leq \delta.
    \end{align*}
    The second claim follows from McDiarmid’s inequality \cite{mcdiarmid1989method}, which tells us that if $Z_S$ is such that $|Z_S-Z_{S'}|\leq c_j$  for any $S=\{x_i\}_{i=1}^n$, $S'=\{x_i'\}_{i=1}^n$ satisfying $x_i=x_i'$ for all $i\neq j$, then we have
    \[
    \Pr[Z_S> \E Z_S+\eta]\leq \exp\left({\frac{-2\eta^2}{\sum_{i=1}^n c_i^2}}\right).
    \]
    Note that in our case this holds with $c_j=\frac{2R}{n}$ for all $j$, hence  we have 
    \[
    Pr[Z_S>\E Z_S+\frac{\varepsilon}{2}]\leq e^{\frac{-\varepsilon^2n}{8R^2}} \leq \delta.
    \]
    So we just need to bound the expectation. For that we use symmetrization
    \begin{align}
        &\EEs{S\sim q^n}{\sup_{f\in \cF}\left|\EEs{x\sim q}{f(x)}-\frac{1}{n}\sum f(x_i)\right|}\\&= \EEs{S\sim q^n}{\sup_{f\in \cF}\left|\EEs{S'\sim q^n}{\frac{1}{n}\sum_{i=1}^nf(x_i')}-\frac{1}{n}\sum f(x_i)\right|} 
        \\&\leq\EEs{S,S'\sim q^n}{\sup_{f\in \cF}\left|{\frac{1}{n}\sum_{i=1}^nf(x_i')}-\frac{1}{n}\sum f(x_i)\right|}.
    \end{align}
    We generate the samples $S=\{x_i\}_{i=1}^n$ and $S'=\{x_i'\}_{i=1}^n$ via the following symmetrization procedure. First draw a sample
\[
T=\{t_i\}_{i=1}^n \cup \{t_i'\}_{i=1}^n
\]
of size $2n$. Then, for each $1\le i\le n$, independently set $(x_i,x_i')=(t_i,t_i')$ with probability $\frac{1}{2}$, and $(x_i,x_i')=(t_i',t_i)$ with probability $\frac{1}{2}$.
This construction yields
\[
\EEs{S,S'\sim q^n}{\sup_{f\in \cF}\left|{\frac{1}{n}\sum_{i=1}^nf(x_i')}-\frac{1}{n}\sum f(x_i)\right|}=\EEs{T,\sigma}{\sup_{f\in\cF} \left|\frac{1}{n}\sum_{i=1}^n \sigma_i\big(f(t_i)-f(t'_i)\big)\right| }
\]
where $\sigma=(\sigma_1,\dots,\sigma_n)\in\{\pm1\}^n$ are independent Rademacher random variables, with $\sigma_i=1$ if $(x_i,x_i')=(t_i,t_i')$ and $\sigma_i=-1$ otherwise. By the triangle inequality and sub-additivity of the supremum we have \[
\EEs{T,\sigma}{\sup_{f\in\cF} \left|\frac{1}{n}\sum_{i=1}^n \sigma_i\big(f(t_i)-f(t'_i)\big)\right| }\leq 2\sup_{x\in \cX^{n}}\EEs{\sigma}{\sup_{f\in\cF}\left|\frac{1}{n}\sum_{i=1}^n \sigma_if(x_i)\right|}.
\] 
The quantity above is known as the Rademacher complexity of $\cF$. To bound fix some $x\in \cX^{n}$ and let $H$ be a cover of $\cF|_x$ of size $N=N(\cF,\frac{\gamma}{2}+\frac{\varepsilon}{8},n)$. Then we may use Massart Lemma (see lemma 26.8 in \cite{understandingml}) for bounding the Rademacher complexity of finite classes to get \[
2\EEs{\sigma}{\sup_{f\in\cF}\left|\frac{1}{n}\sum_{i=1}^n \sigma_if(x_i)\right|}\leq \gamma+\frac{\varepsilon}{4}+2\EEs{\sigma}{\sup_{h\in H}\left|\frac{1}{n}\sum_{i=1}^n \sigma_ih(x_i)\right|}\leq \gamma+\frac{\varepsilon}{4}+\sqrt{\frac{8R^2\log2|H|}{n}}.
\]
This yields the bound 
\[
\EEs{S\sim q^n}{Z_S}\leq  \gamma+\frac{\varepsilon}{4}+\sqrt{\frac{8R^2\cdot H(\cF,\frac{\gamma}{2}+\frac{\varepsilon}{8},n)+8R^2}{n}}\leq \gamma+\frac{\varepsilon}{2}.
\]
Giving the desired result.
\end{proof}

The following result follows from analyzing the proof of Theorem 21 in \cite{BARTLETT1998}, replacing their specific covering bound with a generic one. The proof is exactly the same and is repeated here for completeness.  
\begin{lemma}[Learning Via Covering]\label{lem:lrn-cover}
    Let $\gamma>0$, and let $\cF$ be a bounded function class. Assume that for all $\varepsilon>0$ we have a bound 
    \[
    H(\cF,\gamma+\varepsilon,n)=o(n)
    \]
    on the empirical entropy of $\cF$. Then $\cF$ is $\gamma$-learnable. 
\end{lemma}
\begin{proof}
By normalization and translation assume $\cF\subset [0,1]^\cX$.
    The learning algorithm is as follows: Given a sample $S=\{(x_i,y_i)\}_{i=1}^{2n}$, and a test point $x$ set $X=\{x_i\}_{i=1}^{2n}\cup \{x\}$, and let $H$ be a $(\gamma+\varepsilon)$-cover of $\cF|_X$ of size at most $N(\cF,\gamma+\varepsilon,2n+1)$. Split $S$, as $S_1=\{x_i\}_{i=1}^n$, $S_2=\{x_i\}_{i=n+1}^{2n}$, and let $h\in H$ be such that 
    \[
   L_{S_1}(h):= \frac{1}{n}\sum_{i=1}^n |h(x_i)-y_i|, 
    \]
    is minimized over all functions in $H$. Our learner will return the value of $h$ on the test point \[
    \cA_S(x)=h(x).
    \] 
    To see this gives a valid learner let $\mu$ be a distribution on $\cX\times[-R,R]$, and let $\varepsilon,\delta>0$. Let $f^\star\in \cF$ be a function for which
    \[
    L_\mu(f^\star):=\EEs{(x,y)\sim q}{|f^\star(x)-y|}\leq\inf_{f\in\cF} \EEs{(x,y)\sim q}{|f(x)-y|}+\varepsilon.
    \]
    Now by Hoeffding inequality we have 
    \[
    \Pr_{S\sim \mu^{2n}}[L_{S}(f^\star)\leq L_\mu(f^\star)+\varepsilon]>1-2e^{-4\varepsilon^2n}.
    \]
    On this event, since $H$ is a cover of $\cF|_X$ there is some $h'\in H$ such that \[
    L_{S}(h')\leq L_{S}(f^\star)+\gamma+\varepsilon\leq L_\mu(f^\star)+\gamma+2\varepsilon\leq\inf_{f\in\cF}L_{\mu}(f)+\gamma+3\varepsilon.
    \]
    By symmetrization 
\begin{align*}
    &\Pr_{S_1,S_2\sim \mu^n}[\sup_{h\in H}|L_{S_1}(h)-L_{S_2}(h)|>\varepsilon]\leq 
   \sup_{S\in (\cX\times[0,1] )^{2n}}\Pr_{\sigma}\left[\sup_{h\in H}
    \left|\frac{1}{n}\sum_{i=1}^n \sigma_i\Big(|h(x_i)-y_i|-|h(x_{n+i})-y_{n+i}|\Big) \right|>\varepsilon\right] 
    \\&\leq 2|H|e^{\frac{-\varepsilon^2n}{2}}\leq 2N(\cF,\gamma+\varepsilon,2n+1)e^{\frac{-\varepsilon^2n}{2}}.
\end{align*}
    
    Now since $L_S$ is the average of $L_{S_1}$ and $L_{S_2}$, and $h$ is the minimizer of $L_{S_1}$ we deduce that on the above event
    \[
    L_{S_2}(h)\leq L_{S_1}(h)+\varepsilon\leq L_{S_1}(h')+\varepsilon\leq L_{S}(h')+2\varepsilon\leq\inf_{f\in\cF}L_{\mu}(f)+\gamma+5\varepsilon.
    \]
    Averaging over permutation of the last $n$ points and the test point, which does not affect $h$, yields that the above gives the same bound for $L_\mu(h)$. Hence we get that
    \[
    \Pr \big(L_\mu(h)\geq\inf_{f\in\cF}L_{\mu}(f)+\gamma+5\varepsilon\big)\leq 4N(\cF,\gamma+\varepsilon,2n+1)e^{\frac{-\varepsilon^2n}{2}}.
    \]
    From which the desired result follows.
\end{proof}

\UCintro*

\begin{proof}\label{proof:UCintro}
$(1)\implies(3)$ Follow from Theorem 22 in \cite{BARTLETT1998} and $(2)\implies (3)$ follow from Theorem 26 in the same work. The implication $(3)\implies (1)$ follow from \cref{lem:covering} and \cref{lem:lrn-cover}. 
Thus we only need to prove $(3)\implies (2)$.  Let $\gamma>0$ be such that $\pd_{\gamma'}(\cF)$ is finite for all $\gamma'>\gamma$, and we will show that $\cF$ satisfies $\gamma$-uniform convergence.

Let $\varepsilon,\delta>0$,
 and denote $d=\pd_{\gamma+\frac{\varepsilon}{16}}$, which is finite by assumption. 
By Lemma \ref{lem:covering} we have \[
H\left(\mathcal F,\frac{\gamma}{2}+\frac{\varepsilon}{8},n\right)\leq 14d\log \frac{32R}{\varepsilon}\log ^2n.
\]
Thus, by \cref{lem:uniform-convergence-bound} if $n$ is large enough such that \[
\frac{n}{\log^2 n}> \frac{1792}{\varepsilon}\left(\log \frac{32R}{\varepsilon}
+
\log \frac{1}{\delta}\right)
\]
then we have 
\[
\Pr_{S \sim q^n}\left[
\sup_{f \in \cF}
\left|
\frac{1}{n}\sum_{x \in S} f(x) - \E_{x \sim q}[f(x)]
\right|
> \gamma + \varepsilon
\right]
\leq\delta.
\]
From which the desired result follow. We note that the induced sample size bound will be of order \[
m(\varepsilon,\delta)=O\left(\frac{d\log \frac{1}{\varepsilon}+\log \frac{1}{\delta}}{\varepsilon^2}\log^2\left(\frac{d\log\frac{1}{\varepsilon}+\log \frac{1}{\delta}}{\varepsilon^2}\right)\right).
\]

\end{proof}
\subsection{Proof for \cref{thm:cover-Dich-intro}}
The following Lemma concerns approximate sample compression schemes, which were defined in \cite{hanneke2019sample} as a natural generalization of sample compression schemes for discrete classes. 
We say that $\cF$ has a $\gamma$-sample compression scheme of size $k$, if
there exists a reconstructor 
\[
\rho : (\cX \times \R)^k \to \R^{\cX}
\]
such that for any finite set $X$ and function $f\in \cF$ there exists a sequence of length $k$
\[
T = \{(x_{i}, f(x_{i}))\}_{i=1}^k, \quad x_i\in X \text{ for all $1\leq i\leq k$},
\]
 satisfying
\[
|\rho(T)(x) - f(x)| \le \gamma \quad \text{for all } x\in X.
\]
\begin{lemma}[Sample compression]\label{lem:Sample-compression}
Let $\mathcal F\subseteq [-R,R]^{\mathcal X}$ be a bounded function class, and assume that
$\pd_\gamma(\mathcal F)<\infty$. Then for every $\varepsilon>0$, $\mathcal F$ has a
$(2\gamma+\varepsilon)$-compression scheme of constant size.
\end{lemma}

\begin{proof}
The proof is a real-valued analogue of \cite{moran2016sample}. Fix
$\varepsilon>0$. We will choose $m,m^\star$ below.

For every realizable sample
\[
T=\{(x_i,f(x_i))\}_{i=1}^m,
\]
where $x_i\in\mathcal X$ and $f\in\mathcal F$, choose some
$h_T\in\mathcal F$ such that
\[
h_T(x_i)=f(x_i)\qquad\text{for all }i\in[m].
\]

We will show that for every finite set $X=\{x_1,\ldots,x_n\}\subseteq\mathcal X$
and every $f\in\mathcal F$, there exist subsets
$J_1,\ldots,J_{m^\star}\subseteq X$ of size $m$ such that, writing
\[
T_i=\{(x,f(x)):x\in J_i\},
\]
we have, for every $x\in X$,
\[
\left|
\frac{1}{m^\star}\sum_{i=1}^{m^\star}h_{T_i}(x)-f(x)
\right|
\le 2\gamma+\varepsilon.
\]
The compression scheme stores the samples $T_1,\ldots,T_{m^\star}$ and
reconstructs by
\[
\rho(T_1,\ldots,T_{m^\star})
=
\frac{1}{m^\star}\sum_{i=1}^{m^\star}h_{T_i}.
\]

Fix $X$ and $f$, and define
\[
\mathcal H
=
\left\{
h_T:\ T=\{(x_i,f(x_i))\}_{i=1}^m,\ x_i\in X
\right\}.
\]

We first claim that, for every distribution $q$ over $X$, there exists
$h\in\mathcal H$ such that
\[
\EEs{x\sim q}{|h(x)-f(x)|}
\le
\gamma+\frac{\varepsilon}{2}.
\]
For any function $h :X\to \R$, define
\[
\phi_h(x)=|h(x)-f(x)|,
\]
and let
\[
L_{\mathcal H,f}=\{\phi_h:h\in\mathcal H\}.
\]
For any functions $g,h$ and any $x\in X$,
\[
|\phi_g(x)-\phi_h(x)|
=
\big||g(x)-f(x)|-|h(x)-f(x)|\big|
\le
|g(x)-h(x)|.
\]
Hence, for every finite sample size $n$ and every scale $\alpha>0$,
\[
N(L_{\mathcal H,f},\alpha,n)
\le
N(\mathcal H,\alpha,n)
\le
N(\mathcal F,\alpha,n).
\]
Since $\pd_\gamma(\mathcal F)<\infty$, the covering bound \cref{lem:covering} and the uniform convergence bound
\cref{lem:uniform-convergence-bound} imply that the classes
$L_{\mathcal H,f}$ satisfy $\gamma$-uniform convergence with a sample size
depending only on $\gamma,\varepsilon,R$ and $\pd_\gamma(\mathcal F)$, but not
on $X$, $f$, or $q$. Choose $m$ large enough so that, with probability at least
$2/3$ over $x_1,\ldots,x_m\sim q$, we have simultaneously for all
$h\in\mathcal H$,
\[
\left|
\EEs{x\sim q}{\phi_h(x)}
-
\frac1m\sum_{i=1}^m \phi_h(x_i)
\right|
\le
\gamma+\frac{\varepsilon}{2}.
\]
For the sampled set
\[
T=\{(x_i,f(x_i))\}_{i=1}^m,
\]
the corresponding function $h_T\in\mathcal H$ satisfies
$h_T(x_i)=f(x_i)$ for all $i$. Therefore
\[
\frac1m\sum_{i=1}^m \phi_{h_T}(x_i)=0,
\]
and so
\[
\EEs{x\sim q}{|h_T(x)-f(x)|}
=
\EEs{x\sim q}{\phi_{h_T}(x)}
\le
\gamma+\frac{\varepsilon}{2}.
\]
This proves the claim.

Now consider the zero-sum game in which Player~1 chooses $h\in\mathcal H$,
Player~2 chooses $x\in X$, and the payoff is
\[
|h(x)-f(x)|.
\]
The claim says that for every mixed strategy $q$ of Player~2, there is
$h\in\mathcal H$ whose expected payoff is at most
$\gamma+\frac{\varepsilon}{2}$. By minimax, there exists a distribution $p$
over $\mathcal H$ such that, for every distribution $q$ over $X$,
\[
\EEs{h\sim p,\ x\sim q}{|h(x)-f(x)|}
\le
\gamma+\frac{\varepsilon}{2}.
\]
In particular, for every $x\in X$,
\[
\left|
\EEs{h\sim p}{h(x)}-f(x)
\right|
\le
\EEs{h\sim p}{|h(x)-f(x)|}
\le
\gamma+\frac{\varepsilon}{2}.
\]

Finally, by \cref{prop:DualUC}, the dual class $\cH^\star$ satisfies $\gamma$-uniform
convergence. Choose $m^\star$ large enough so that, with probability at least
$2/3$ over $h_1,\ldots,h_{m^\star}\sim p$,
\[
\sup_{x\in X}
\left|
\EEs{h\sim p}{h(x)}
-
\frac1{m^\star}\sum_{i=1}^{m^\star}h_i(x)
\right|
\le
\gamma+\frac{\varepsilon}{2}.
\]
Therefore there exist $h_1,\ldots,h_{m^\star}\in\mathcal H$ satisfying this
bound. Writing each $h_i$ as $h_{T_i}$, we obtain, for every $x\in X$,
\[
\left|
\frac1{m^\star}\sum_{i=1}^{m^\star}h_{T_i}(x)-f(x)
\right|
\le
\gamma+\frac{\varepsilon}{2}
+
\gamma+\frac{\varepsilon}{2}
=
2\gamma+\varepsilon.
\]

Since $m$ and $m^\star$ depend only on $\gamma,\varepsilon,R$ and the relevant
fat-shattering bounds, and not on $X$, $f$, or $n$, the resulting compression
scheme has constant size $mm^\star$. We remark that our uniform convergence bound imply the compression size is of order \[
mm^\star=O\left( \frac{\pd_\gamma(\cF)\pd_\gamma(\cF^\star)}{\varepsilon^4}\log^6 \frac{\pd_\gamma(\cF)\pd_\gamma(\cF^\star)}{\varepsilon}\right).
\]
\end{proof}

\coverDichIntro*
\begin{proof}\label{proof:coverDichIntro}
The first part  is well established \cite{Alon1997-scale-sensitive,BARTLETT1998}. If $S$ is a $\gamma'$-shattered set then the $\gamma$-covering number of $\cF|_S$ is at least $2^n$ for any $\gamma<\frac{\gamma'}{2}$. 
For the second part , let $\gamma>\frac{\gamma^\star}{2}$ and pick some $\gamma'$ such that $\gamma>\gamma'>\frac{\gamma^\star}{2}$. By definition $\pd_{2\gamma'}(\cF)$ is finite hence by Lemma $\ref{lem:covering}$ we
have \[
N(\cF,\gamma,n)\leq n^{14\pd_{2\gamma'}(\cF)\log \frac{1}{\gamma-\gamma'}\log n}.
\] 
Taking $\log$ gives the desired bound.
 To show this bound can be tight in the $\gamma\in(\frac{\gamma^\star}{2},\gamma^\star)$ regime  we use Theorem 11 of \cite{alon2022theory}, which gives, for any $n>1$,  the existence of a partial concept class $\cF_n\subset \{-1,1,\star\}^n$, with VC dimension $1$, such that any disambiguation of $\cF_n$ is of size at least $n^{(\log n)^{1-o(1)}}$.

    For each $f\in \cF_n$ define $\hat{f}:\cX\to [-\gamma^\star,\gamma^\star]$ by 
    \[
    \hat{f}(x)=\begin{cases}
        \gamma^\star \quad &f(x)=1,
        \\-\gamma^\star \quad & f(x)=-1,
        \\0 \quad & f(x)=\star,
    \end{cases}
    \]
    And set $\Hat{\cF}_n=\{\hat{f}\;:\; f\in \cF_n\}$. 
     Then $\pd_\gamma(\Hat{\cF}_n)$ is finite (And equal to $1$)  if $\gamma>\gamma^\star$, other properties of this class will imply this is an if and only if but this is not necessary for our bound. Now let $H$ be some $\gamma$-cover of $\Hat{\cF}_n$, for $\gamma<\gamma^\star$. Define $\{\Phi_h\;:\; h\in H\}$ a disambiguation of $\cF_n$ by \[
    \Phi_h(x)=\begin{cases}
        1\quad &h(x)\geq0,
        \\-1 \quad &h(x)< 0.
    \end{cases}
    \] 
    For any $f\in \cF_n$ there is $h\in H$ such that $\norm{\hat{f}-h}_\infty\leq \gamma$, hence we have $h(x)\geq \gamma^\star-\gamma>0$, whenever $f(x)=1$ and $h(x)\leq\gamma-\gamma^\star<0$ whenever $f(x)=-1$. Hence $\phi_h$ disambiguates $f$ and $\{\phi_h\;:\; h\in H\}$ is a disambiguation of $\cF_n$, implying that its size is at least $n^{(\log n)^{1-o(1)}}$. This gives the result for a fixed $n$, to get $\cF$ for which the bound will hold for all $n$ simply take a disjoint union of all the $\cF_n$. each defined over its own domain and extend to be $0$ on all other domains.

For the $\log n$ bound let $\gamma>2\gamma^\star$ and define $\varepsilon=\frac{\gamma-2\gamma^\star}{8}$ so $\gamma=2\gamma^\star+8\varepsilon$. First discretize the space at scale $\varepsilon$, replacing each function $f\in \cF$ with $f_\varepsilon$ defined by \[
f_\varepsilon(x)=\floor{\frac{f(x)}{\varepsilon}}\varepsilon,
\]
and Set $\cF_\varepsilon=\{f_\varepsilon\;:\; f\in \cF\}$.  Note that $|f(x)-f_\varepsilon(x)|\leq \varepsilon$ for all $x\in \cX$,
hence we have
\[
\pd_{\frac{\gamma}{2}-\varepsilon}(\cF_\varepsilon)\leq \pd_{\frac{\gamma}{2}-3\varepsilon}(\cF)=\pd_{\gamma^\star+\varepsilon}(\cF)<\infty.
\]
Thus, by Theorem \ref{thm:UCintro} we have that $\cF_\varepsilon$ satisfies $(\frac{\gamma}{2}-\varepsilon)$-uniform convergence. By Proposition \ref{prop:DualUC} we also have that its dual class $\cF_\varepsilon^\star$ satisfies  $(\frac{\gamma}{2}-\varepsilon)$-uniform convergence. Hence by Lemma \ref{lem:Sample-compression} $\cF_\varepsilon$ admits a $(\gamma-\varepsilon)$-sample compression of constant size $k$, i.e. there is a reconstructor $\rho$, such that for any dataset $S=\{(x_i,f(x_i))\}_{i=1}^n$ there is $T=\{\big(x_{i_j},f(x_{i_j})\big)\}_{i=1}^k$ such that the reconstructed function $\rho(T)$ is at most $(\gamma-\varepsilon)$-away from $f$ on any point in $S$. We use this compression scheme  to construct a covering as follow : given a dataset $X$ of size $n$ define the covering $H$ by 
\[
H=\{\rho(T)\;:\; T=\{(x_i,f_\varepsilon(x_i)\}_{i=1}^k, f_\varepsilon\in \cF_\varepsilon,\; x_i\in X\}.
\]
So $H$ is all possible outputs of reconstructor on realizable subset of $X$ of size $k$. Note that each function in $H$ is defined by $k$ points from $X$ and the values of $f_\varepsilon$ on them. By discretization $f_\varepsilon$ has no more than $\frac{2R}{\varepsilon}$ possible outputs at every point, yielding a bound of
\[
|H|\leq (\frac{2Rn}{\varepsilon})^k.
\]
So it just remains to show that $H$ is an $\gamma$-cover for $\cF$.
By definition of compression scheme $H$ forms a $(\gamma-\varepsilon)$-cover of $\cF_\varepsilon$, and since $|f(x)-f_\varepsilon(x)|\leq \varepsilon$ for all $x$ we indeed have that it is also a  $\gamma$ cover of  $\cF$. Taking $\log$, and noting that $k,\varepsilon$ are independent on $n$ gives the desired result.

To see this $\log n$ bound can be tight let $\cF\subset \{-1,1\}^\NN$ be the class of all singleton functions, so $\cF=\{f_n\;:\; n\in \NN\}$,
where $f_n(k)$ is $1$ if $n=k$, and $-1$ otherwise. It is easy to see that $\pd_\gamma(\cF)=1$, and $N(\cF,\gamma,n)=n$ for all $\gamma\in(0,1)$.
\end{proof}

\begin{prop}\label{prop:crit-val-cover}
    Let $\gamma^\star>0$. There exists bounded function classes $\cF_1,\cF_2$ such that \[
    \inf\{\gamma\;:\; \pd_\gamma(\cF_1)<\infty\}=\gamma^\star=\inf\{\gamma\;:\; \pd_\gamma(\cF_2)<\infty\}.
    \] 
    But for all $n>0$,  we have 
    \begin{align*}
        &H(\cF_1,\frac{\gamma^\star}{2},n)=0,
        \\&H(\cF_2,\frac{\gamma^\star}{2},n)=\Omega(\frac{n}{\log n}).
    \end{align*}
\end{prop}
\begin{proof}
    Let $\cF_1=[-\frac{\gamma^\star}{2},\frac{\gamma^\star}{2}]^\NN$. Clearly we have $\pd_\gamma(\cF_1)=0$ for $\gamma>\gamma^\star$ and  $\pd_\gamma(\cF_1)=\infty$ for $\gamma\leq \gamma^\star$ so $\gamma^\star$ is indeed the critical threshold for finite fat shattering. It is also clear that $N(\cF_1,\frac{\gamma^\star}{2},n)=1$ for all $n$ since the zero function is distance $\frac{\gamma^\star}{2}$ from any function in $\cF_1$,

    Now let $\cF_2\subset [0,1+\gamma^\star]^{\NN}$ be the class of all function $f$ from $\NN$ to $[0,1+\gamma^\star]$ satisfying \[
    f(n)\leq \gamma^\star+\frac{1}{\log n}.
    \]
    It is simple to see that $\pd_{\gamma}(\cF_2)>n$ if and only if $\gamma\leq \gamma^\star+\frac{1}{\log n}$, thus $\gamma^\star$ is indeed the critical threshold for finite fat shattering. Note that the restriction of $\cF_2$ to $1,2\dots n$, contains $[0,\gamma^\star+\frac{1}{\log n}]^n$ as a subset. It has a volume of $(\gamma^\star+\frac{1}{\log n})^n$, hence if we want to cover it with balls of radius $\frac{\gamma^\star}{2}$, each of volume $(\gamma^\star)^n$ the number of functions we must use is at least \[
   \left(\frac{\gamma^\star+\frac{1}{\log n}}{\gamma^\star}\right)^n=\left(1+\frac{1}{\gamma^\star \log n}\right)^n\sim e^{\frac{n}{\gamma^\star\log n}}.
    \]
    This shows that $H(\cF_2,\frac{\gamma^\star}{2},n)=\Omega(\frac{n}{\log n})$. Note that by replacing $\frac{1}{\log n}$ with a slower decaying sequence we may get arbitrarily close to a $\log n$ lower bound.
\end{proof}

\subsection{Composition and Closure}\label{sec:compandclosure}
\begin{lemma}\label{lem:fixedthreshold}
Let $\cF$ be a bounded function class. Let \(\gamma^\star=\inf\{\gamma\;:\; \pd_\gamma(\cF)<\infty\}
\). 
% Let $\gamma^\star>0$, and let $\cF$ be a bounded function class with $\pd_{\gamma^\star}(\cF)=\infty$. 
Then for any $\gamma^-<\gamma^\star$, and any $n\in \NN$, there exists $r\in \R$ such that there are at least $n$ points that can be $\gamma^-$-shattered at the same threshold $r$ by $\cF$.
\end{lemma}

\begin{proof}
By rescaling and translating $\cF$, we assume $\cF\subset [0,1]^\cX$.
    Fix $n \in \NN$ and let $M \in \NN$ be chosen later. Fix a scale $\gamma'$ such that $\gamma^- < \gamma' < \gamma^\star$. By definition, there exists a subset $\{x_1,...,x_M \} \subseteq \cX$ that is $\gamma'-$shattered by $\cF$. with thresholds $r_1,...,r_M$. Observe that since $\cF \subseteq [0,1]^\cX$, any threshold $r_i \in [\gamma'/2, 1-\gamma'r/2]$. Now, partition the interval $[\gamma'/2,1-\gamma'/2]$ into subintervals of length $\alpha \leq \frac{\gamma' - \gamma^-}{2}$. Note that there can be at most $S = \lceil \frac{1-\gamma'}{\alpha} \rceil$ subintervals and that we can identify any subinterval with an index set $I \subseteq [M]$. It follows by the pigeonhole principle that if we take $M \geq Sn$, there must exist an index set $I^\star \subseteq [M]$ with $|I^\star| \geq n$ for which there exists $r \in \mathbb{R}$ such that $$r_i \in [r, r+\alpha] \hspace{1em} \forall i \in I^\star.$$ 
    
    Now, we claim that $\{x_i : i \in I^\star \}$ is $\gamma^-$-shattered by $\cF$ with common threshold $r$. To see this, fix any labeling $y \in \{+1,-1\}^{M}$ (taking the subset of labelings corresponding to $I^\star$ gives an arbitrary labeling on $I^\star$). By definition, there exists $f \in \cF$ such that for all $i \in I^\star$: 

    If $y_i = +1$, then $$f(x_i) \geq r_i + \frac{\gamma'}{2} \geq r + \frac{\gamma'}{2}\geq r + \frac{\gamma^-}{2}$$
    If $y_i = -1$, then $$f(x_i) \leq r_i - \frac{\gamma'}{2} \leq  r + \alpha - \frac{\gamma'}{2} \leq r - \frac{\gamma^-}{2},$$
    
    where the last inequality uses the fact that $\alpha \leq \frac{\gamma' - \gamma^-}{2}$. The result follows.
\end{proof}

\DualUC*
\begin{proof}\label{proof:DualUC}
     By Theorem \ref{thm:UCintro} it is enough to show that for any $\gamma'>\gamma$ we have that $\pd_{\gamma'}(\cF^\star)<\infty$. Indeed, let $\gamma'>\gamma$ and choose some $\gamma''$ such that $\gamma<\gamma''<\gamma'$. Assume toward contradiction that $\pd_{\gamma'}(\cF^\star)=\infty$, and 
     denote $n=2^{\pd_{\gamma''}(\cF)+2}$, which is finite by Theorem \ref{thm:UCintro}. By Lemma \ref{lem:fixedthreshold} there is some $r\in \R$ such that there are $n$ points which  are $\gamma''$-shattered by $\cF^\star$ at this same threshold $r$. Now that the points are shattered  at the same threshold, standard Assouad type bounds, such as Theorem 3.7 in \cite{KLEER2023-Dual}, will imply that $\cF$  $\gamma''$-shatters at least $\log (n)-1=\pd_{\gamma''}(\cF)+1$ points.  Giving the desired contradiction. 
\end{proof}

\Aggregation*
\begin{proof}\label{proof:aggregation}
    Let $\varepsilon>0$.
    By Theorem \ref{thm:UCintro} we know that $\cF_i$ has finite $(\gamma+\varepsilon)$-fat-shattering dimension for each $1\leq i\leq k$. Denote $d_i=\pd_{\gamma+\varepsilon}(\cF_i)$, then by Lemma \ref{lem:covering} we have 
    \[
    N(\cF_i,\frac{\gamma}{2}+\varepsilon)\leq n^{14 d_i\log \frac{1}{\varepsilon}\log n}
    \]
    For each $1\leq i\leq k$ let $H_i$ be such $(\frac{\gamma}{2}+\varepsilon)$-cover for $\cF_i$ and define \[
    G(H)=\{G(h_1,h_2 \dots h_k)\;:\;\forall i,\; h_i \in H_i \}.
    \]
    We claim that $G(H)$ is a $(\frac{\gamma}{2}+\varepsilon)$-cover for $G(\cF)$. Indeed, for any $f=(f_1,f_2,\dots f_k)\in \cF$ there is $h=(h_1,h_2, \dots h_k)$, such that $h_i\in H_i$, and  $\norm{f_i-h_i}_\infty\leq \frac{\gamma}{2}+\varepsilon$ for all $i$. Then since $G$ is $1$-Lipschitz with respect to the $\ell_\infty$ norm we have that 
    \[
    \norm{G(f)-G(h)}_\infty\leq \norm{f-h}_\infty=\sup_{1\leq i\leq k}\sup_{x\in \cX} |f_i(x)-h_i(x)|\leq \frac{\gamma}{2}+\varepsilon.
    \]
    Thus we get that
    \[
    N(G(\cF),\frac{\gamma}{2}+\varepsilon,n)\leq\prod _{i=1}^k N(\cF_i,\frac{\gamma}{2}+\varepsilon,n)\leq \exp\left(14\log^2 n\log \frac{1}{\varepsilon}\sum_{i=1}^k d_i\right).
    \]
   The above with 
     Lemma \ref{lem:uniform-convergence-bound} implies the desired result.
\end{proof}

\section{Proofs of Evaluability Results} \label{sec:evaluabilitythm}
\subsection{Lower bound construction of \cite{bousquet2019optimal}}
\begin{lemma}[Lemma 22 of \cite{bousquet2019optimal}]\label{lem:bkm22}
Let $\mathcal{D}_1$ and $\mathcal{D}_2$ be two families of probability distributions, $\mathcal{D}_i^{\oplus m}$ denotes the distribution obtained by sampling $p \sim \mathcal{D}_i$ (assuming some given fixed distribution over $\mathcal{D}_i$) and then drawing $m$ independent samples from $p$. Consider an algorithm (which can be randomized) that determines, given $m$ i.i.d. examples from some $p \in \mathcal{D}_1 \cup \mathcal{D}_2$, whether $p \in \mathcal{D}_1$ or $p \in \mathcal{D}_2$. Then such an algorithm will have a probability of making a mistake lower bounded by
\[
\frac{1}{2}\left(1 - \operatorname{TV}(\mathcal{D}_1^{\oplus m}, \mathcal{D}_2^{\oplus m})\right).
\]
\end{lemma}

\begin{proof}
We first assume that the algorithm is deterministic. Any deterministic algorithm deciding whether $p$ comes from $\mathcal{D}_1$ or $\mathcal{D}_2$ is associated with a set $A \subseteq \mathcal{X}^m$ (the set such that if the sample falls in it, it decides $i = 1$, and $i = 2$ otherwise). The worst-case probability of the algorithm to err is given by
\[
\max\left(\max_{p \in \mathcal{D}_2} p^m(A), \, \max_{p \in \mathcal{D}_1} p^m(\bar{A})\right)
\]
which can be lower bounded by the expectation under first choosing between $i = 1$ and $i = 2$ with probability $1/2$ and then picking $p \sim \mathcal{D}_i$:
\begin{align*}
\frac{1}{2}\left(\mathbb{E}_{p \sim \mathcal{D}_1} p^m(A) + \mathbb{E}_{p \sim \mathcal{D}_2} p^m(\bar{A})\right)
&= \frac{1}{2}\left(1 + \mathcal{D}_1^{\oplus m}(A) - \mathcal{D}_2^{\oplus m}(A)\right) \\
&\geq \frac{1}{2}\left(1 - \operatorname{TV}(\mathcal{D}_1^{\oplus m}, \mathcal{D}_2^{\oplus m})\right).
\end{align*}
If the algorithm is randomized then it may pick $A$ randomly, so there is an additional expectation with respect to the distribution over sets $A$ which also leads to the same lower bound.
\end{proof}

\begin{lemma}[Lemma 23 of \cite{bousquet2019optimal}]\label{lem:bkm23}
Given two probability distributions $P,Q$ on a domain $\mathcal{X}$ and an event $E \subseteq \mathcal{X}$, denoting by $P_{\mid E}$ and $Q_{\mid E}$ the corresponding conditional distributions (i.e. $P_{\mid E}(A) := P(A\mid E)$), we have
\[
\operatorname{TV}(P,Q) \leq \operatorname{TV}(P_{\mid E},Q_{\mid E}) + 2P(\overline{E}) + 2Q(\overline{E}).
\]
\end{lemma}

\begin{proof}
\begin{align*}
\operatorname{TV}(P,Q)
&= \sup_A |P(A)-Q(A)| \\
&\leq \sup_A |P(A \cap E)-Q(A \cap E)| + \sup_A |P(A \cap \overline{E})-Q(A \cap \overline{E})| \\
&\leq \sup_A \left|P(E)\bigl(P(A\mid E)-Q(A\mid E)\bigr) + Q(A\mid E)\bigl(P(E)-Q(E)\bigr)\right| \\
&\quad + P(\overline{E}) + Q(\overline{E}) \\
&\leq P(E) \sup_A |P(A\mid E)-Q(A\mid E)| + |P(E)-Q(E)| \\
&\quad + P(\overline{E}) + Q(\overline{E}) \\
&\leq \operatorname{TV}(P_{\mid E},Q_{\mid E}) + |P(E)-Q(E)| + P(\overline{E}) + Q(\overline{E}) \\
&= \operatorname{TV}(P_{\mid E},Q_{\mid E}) + |P(\overline{E})-Q(\overline{E})| + P(\overline{E}) + Q(\overline{E}) \\
&\leq \operatorname{TV}(P_{\mid E},Q_{\mid E}) + 2P(\overline{E}) + 2Q(\overline{E}).
\end{align*}
\end{proof}

\paragraph{Lower bound construction.} We now present the construction of \cite{bousquet2019optimal}.
Fix $N\in \NN,\beta \in (0,1)$ such that $\frac{1+\beta}{\beta} \in \NN$ and let $k = \frac{1+\beta}{\beta}$. Given a collection of $2kN$ non-overlapping distributions over $\cX$, which we denote by $\bmu =\{\mu_{i,j} |i\in [2N], j\in [k]\}$, we construct mixtures over $\bmu$.
We refer to distributions in $\bmu$ as elements to avoid ambiguity.
% Note that in \cite{bousquet2019optimal}, $\mu_{i,j}$ refers to a single point rather than a distribution.

We define hypothesis class $\cH = \{q_1, q_2\}$:
\begin{itemize}
    \item $q_1$ assigns probability $\frac{1-\beta}{2kN}$ to each $\mu_{i,j}$ and $\frac{1+\beta}{2kN}$ to each $\mu_{i+N,j}$ for all $i\in [N]$ and $j\in [k]$.
    \item $q_2$ assigns probability $\frac{1+\beta}{2kN}$ to each $\mu_{i,j}$ and $\frac{1-\beta}{2kN}$ to each $\mu_{i+N,j}$ for all $i\in [N]$ and $j\in [k]$.
\end{itemize}

For each $i\in [N]$, we pick one entry by an index function $l: [2N]\mapsto [k]$ and then define distributions that concentrate all mass on that picked entries:
\begin{itemize}
    \item $p_{1,l}$: Starting from $q_1$, for all $i \in [N]$, we increase the probability of $\mu_{i,l(i)}$ from $\frac{1-\beta}{2kN}$ to $\frac{1}{kN}$ and set the probability of $\mu_{i+N,l(i+N)}$ to $0$. All other elements retain their probabilities from $q_1$. 
    \item $p_{2,l}$: Starting from $q_2$, for all $i \in [N]$, we increase the probability of $\mu_{i+N,l(i+N)}$ from $\frac{1-\beta}{2kN}$ to $\frac{1}{kN}$ and set the probability of $\mu_{i,l(i)}$ to $0$. All other elements retain their probabilities from $q_2$.
\end{itemize}

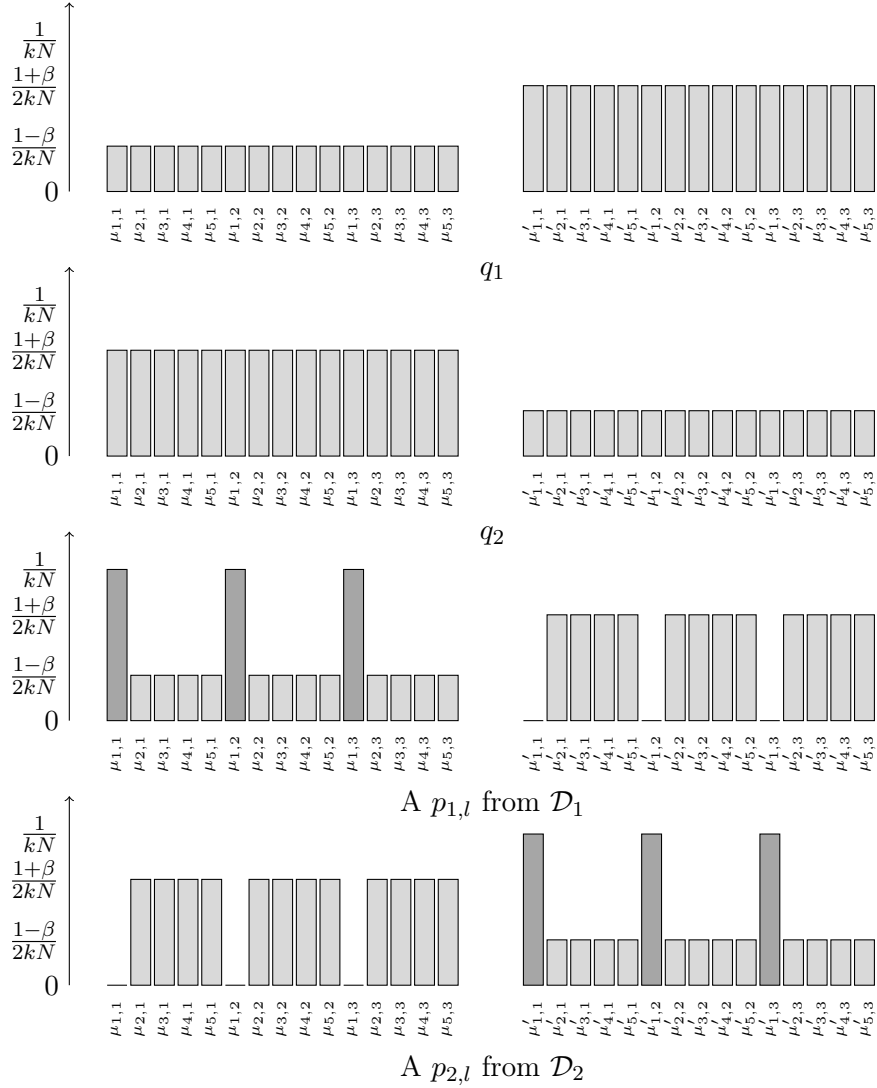
\begin{figure}[h]
\centering
\begin{tikzpicture}[scale=1.0]
% Define parameters - adjusted to match the image exactly
\def\barwidth{0.2625}
\def\barsep{0.05}
\def\halfwidth{4.6875}  % Width of each half (15 bars * (0.2625 + 0.05) = 4.6875)
\def\leftstart{0}
\def\rightstart{5.5}
\def\qstart{3.5}
\def\qtwostart{0}
\def\pstart{-3.5}
\def\p2start{-7.0}

% First chart: q_1
% Left half - all bars at height 1-β
% Bars are indexed as μ_{1,1}, μ_{2,1}, ..., μ_{5,1}, μ_{1,2}, ..., μ_{5,2}, μ_{1,3}, ..., μ_{5,3}
\foreach \i in {0,1,2,3,4,5,6,7,8,9,10,11,12,13,14} {
    \pgfmathsetmacro{\x}{\leftstart + \i*(\barwidth+\barsep)}
    \pgfmathparse{int(mod(\i,5))}
    \let\modfive\pgfmathresult
    \pgfmathparse{int(floor(\i/5))}
    \let\j\pgfmathresult
    \pgfmathparse{int(\modfive+1)}
    \edef\firstidx{\pgfmathresult}
    \pgfmathparse{int(\j+1)}
    \edef\secondidx{\pgfmathresult}
    % Light gray bar (height 1-β)
    \filldraw[fill=gray!30, draw=black] (\x, \qstart) rectangle (\x+\barwidth, \qstart+0.6);
    \node[font=\tiny, rotate=90, anchor=east] at (\x+\barwidth/2, \qstart-0.05) {$\mu_{\firstidx,\secondidx}$};
}

% Right half - all bars at height 1+β
% Bars are indexed as μ'_{1,1}, μ'_{2,1}, ..., μ'_{5,1}, μ'_{1,2}, ..., μ'_{5,2}, μ'_{1,3}, ..., μ'_{5,3}
\foreach \i in {0,1,2,3,4,5,6,7,8,9,10,11,12,13,14} {
    \pgfmathsetmacro{\x}{\rightstart + \i*(\barwidth+\barsep)}
    \pgfmathparse{int(mod(\i,5))}
    \let\modfive\pgfmathresult
    \pgfmathparse{int(floor(\i/5))}
    \let\j\pgfmathresult
    \pgfmathparse{int(\modfive+1)}
    \edef\firstidx{\pgfmathresult}
    \pgfmathparse{int(\j+1)}
    \edef\secondidx{\pgfmathresult}
    % Light gray bar (height 1+β)
    \filldraw[fill=gray!30, draw=black] (\x, \qstart) rectangle (\x+\barwidth, \qstart+1.4);
    \node[font=\tiny, rotate=90, anchor=east] at (\x+\barwidth/2, \qstart-0.05) {$\mu'_{\firstidx,\secondidx}$};
}

% Y-axis for q_1 chart
\draw[->] (-0.5, \qstart) -- (-0.5, \qstart+2.5);
\node[left] at (-0.5, \qstart) {0};
\node[left] at (-0.5, \qstart+0.6) {$\frac{1-\beta}{2kN}$};
\node[left] at (-0.5, \qstart+1.4) {$\frac{1+\beta}{2kN}$};
\node[left] at (-0.5, \qstart+2) {$\frac{1}{kN}$};

% Second chart: q_2
% Left half - all bars at height 1+β
% Bars are indexed as μ_{1,1}, μ_{2,1}, ..., μ_{5,1}, μ_{1,2}, ..., μ_{5,2}, μ_{1,3}, ..., μ_{5,3}
\foreach \i in {0,1,2,3,4,5,6,7,8,9,10,11,12,13,14} {
    \pgfmathsetmacro{\x}{\leftstart + \i*(\barwidth+\barsep)}
    \pgfmathparse{int(mod(\i,5))}
    \let\modfive\pgfmathresult
    \pgfmathparse{int(floor(\i/5))}
    \let\j\pgfmathresult
    \pgfmathparse{int(\modfive+1)}
    \edef\firstidx{\pgfmathresult}
    \pgfmathparse{int(\j+1)}
    \edef\secondidx{\pgfmathresult}
    % Light gray bar (height 1+β)
    \filldraw[fill=gray!30, draw=black] (\x, \qtwostart) rectangle (\x+\barwidth, \qtwostart+1.4);
    \node[font=\tiny, rotate=90, anchor=east] at (\x+\barwidth/2, \qtwostart-0.05) {$\mu_{\firstidx,\secondidx}$};
}

% Right half - all bars at height 1-β
% Bars are indexed as μ'_{1,1}, μ'_{2,1}, ..., μ'_{5,1}, μ'_{1,2}, ..., μ'_{5,2}, μ'_{1,3}, ..., μ'_{5,3}
\foreach \i in {0,1,2,3,4,5,6,7,8,9,10,11,12,13,14} {
    \pgfmathsetmacro{\x}{\rightstart + \i*(\barwidth+\barsep)}
    \pgfmathparse{int(mod(\i,5))}
    \let\modfive\pgfmathresult
    \pgfmathparse{int(floor(\i/5))}
    \let\j\pgfmathresult
    \pgfmathparse{int(\modfive+1)}
    \edef\firstidx{\pgfmathresult}
    \pgfmathparse{int(\j+1)}
    \edef\secondidx{\pgfmathresult}
    % Light gray bar (height 1-β)
    \filldraw[fill=gray!30, draw=black] (\x, \qtwostart) rectangle (\x+\barwidth, \qtwostart+0.6);
    \node[font=\tiny, rotate=90, anchor=east] at (\x+\barwidth/2, \qtwostart-0.05) {$\mu'_{\firstidx,\secondidx}$};
}

% Y-axis for q_2 chart
\draw[->] (-0.5, \qtwostart) -- (-0.5, \qtwostart+2.5);
\node[left] at (-0.5, \qtwostart) {0};
\node[left] at (-0.5, \qtwostart+0.6) {$\frac{1-\beta}{2kN}$};
\node[left] at (-0.5, \qtwostart+1.4) {$\frac{1+\beta}{2kN}$};
\node[left] at (-0.5, \qtwostart+2) {$\frac{1}{kN}$};

% Third chart: p ∈ D1
% Left half - every 5th bar is tall (bars at positions 0, 5, 10 in 0-indexed, i.e., 1st, 6th, 11th)
% Bars are indexed as μ_{1,1}, μ_{2,1}, ..., μ_{5,1}, μ_{1,2}, ..., μ_{5,2}, μ_{1,3}, ..., μ_{5,3}
\foreach \i in {0,1,2,3,4,5,6,7,8,9,10,11,12,13,14} {
    \pgfmathsetmacro{\x}{\leftstart + \i*(\barwidth+\barsep)}
    \pgfmathparse{int(mod(\i,5))}
    \let\modfive\pgfmathresult
    \pgfmathparse{int(floor(\i/5))}
    \let\j\pgfmathresult
    \pgfmathparse{int(\modfive+1)}
    \edef\firstidx{\pgfmathresult}
    \pgfmathparse{int(\j+1)}
    \edef\secondidx{\pgfmathresult}
    \ifnum\modfive=0
        % Dark gray bar (height 2) - every 5th bar (bars 1, 6, 11 in 1-indexed)
        \filldraw[fill=gray!70, draw=black] (\x, \pstart) rectangle (\x+\barwidth, \pstart+2);
        \node[font=\tiny, rotate=90, anchor=east] at (\x+\barwidth/2, \pstart-0.05) {$\mu_{\firstidx,\secondidx}$};
    \else
        % Light gray bar (height 1-β)
        \filldraw[fill=gray!30, draw=black] (\x, \pstart) rectangle (\x+\barwidth, \pstart+0.6);
        \node[font=\tiny, rotate=90, anchor=east] at (\x+\barwidth/2, \pstart-0.05) {$\mu_{\firstidx,\secondidx}$};
    \fi
}

% Right half - every 5th bar is zero, others are at height 1+β
% Bars are indexed as μ'_{1,1}, μ'_{2,1}, ..., μ'_{5,1}, μ'_{1,2}, ..., μ'_{5,2}, μ'_{1,3}, ..., μ'_{5,3}
\foreach \i in {0,1,2,3,4,5,6,7,8,9,10,11,12,13,14} {
    \pgfmathsetmacro{\x}{\rightstart + \i*(\barwidth+\barsep)}
    \pgfmathparse{int(mod(\i,5))}
    \let\modfive\pgfmathresult
    \pgfmathparse{int(floor(\i/5))}
    \let\j\pgfmathresult
    \pgfmathparse{int(\modfive+1)}
    \edef\firstidx{\pgfmathresult}
    \pgfmathparse{int(\j+1)}
    \edef\secondidx{\pgfmathresult}
    \ifnum\modfive=0
        % Zero height bar (every 5th bar)
        \draw[black] (\x, \pstart) -- (\x+\barwidth, \pstart);
        \node[font=\tiny, rotate=90, anchor=east] at (\x+\barwidth/2, \pstart-0.05) {$\mu'_{\firstidx,\secondidx}$};
    \else
        % Light gray bar (height 1+β)
        \filldraw[fill=gray!30, draw=black] (\x, \pstart) rectangle (\x+\barwidth, \pstart+1.4);
        \node[font=\tiny, rotate=90, anchor=east] at (\x+\barwidth/2, \pstart-0.05) {$\mu'_{\firstidx,\secondidx}$};
    \fi
}

% Y-axis for p_{1,1} chart
\draw[->] (-0.5, \pstart) -- (-0.5, \pstart+2.5);
\node[left] at (-0.5, \pstart) {0};
\node[left] at (-0.5, \pstart+0.6) {$\frac{1-\beta}{2kN}$};
\node[left] at (-0.5, \pstart+1.4) {$\frac{1+\beta}{2kN}$};
\node[left] at (-0.5, \pstart+2) {$\frac{1}{kN}$};

% Fourth chart: p ∈ D2
% Left half - every 5th bar is zero, others are at height 1+β (same as right half of top plot)
% Bars are indexed as μ_{1,1}, μ_{2,1}, ..., μ_{5,1}, μ_{1,2}, ..., μ_{5,2}, μ_{1,3}, ..., μ_{5,3}
\foreach \i in {0,1,2,3,4,5,6,7,8,9,10,11,12,13,14} {
    \pgfmathsetmacro{\x}{\leftstart + \i*(\barwidth+\barsep)}
    \pgfmathparse{int(mod(\i,5))}
    \let\modfive\pgfmathresult
    \pgfmathparse{int(floor(\i/5))}
    \let\j\pgfmathresult
    \pgfmathparse{int(\modfive+1)}
    \edef\firstidx{\pgfmathresult}
    \pgfmathparse{int(\j+1)}
    \edef\secondidx{\pgfmathresult}
    \ifnum\modfive=0
        % Zero height bar (every 5th bar)
        \draw[black] (\x, \p2start) -- (\x+\barwidth, \p2start);
        \node[font=\tiny, rotate=90, anchor=east] at (\x+\barwidth/2, \p2start-0.05) {$\mu_{\firstidx,\secondidx}$};
    \else
        % Light gray bar (height 1+β)
        \filldraw[fill=gray!30, draw=black] (\x, \p2start) rectangle (\x+\barwidth, \p2start+1.4);
        \node[font=\tiny, rotate=90, anchor=east] at (\x+\barwidth/2, \p2start-0.05) {$\mu_{\firstidx,\secondidx}$};
    \fi
}

% Right half - every 5th bar is tall (bars at positions 0, 5, 10 in 0-indexed, i.e., 1st, 6th, 11th)
% Bars are indexed as μ'_{1,1}, μ'_{2,1}, ..., μ'_{5,1}, μ'_{1,2}, ..., μ'_{5,2}, μ'_{1,3}, ..., μ'_{5,3}
\foreach \i in {0,1,2,3,4,5,6,7,8,9,10,11,12,13,14} {
    \pgfmathsetmacro{\x}{\rightstart + \i*(\barwidth+\barsep)}
    \pgfmathparse{int(mod(\i,5))}
    \let\modfive\pgfmathresult
    \pgfmathparse{int(floor(\i/5))}
    \let\j\pgfmathresult
    \pgfmathparse{int(\modfive+1)}
    \edef\firstidx{\pgfmathresult}
    \pgfmathparse{int(\j+1)}
    \edef\secondidx{\pgfmathresult}
    \ifnum\modfive=0
        % Dark gray bar (height 2) - every 5th bar (bars 1, 6, 11 in 1-indexed)
        \filldraw[fill=gray!70, draw=black] (\x, \p2start) rectangle (\x+\barwidth, \p2start+2);
        \node[font=\tiny, rotate=90, anchor=east] at (\x+\barwidth/2, \p2start-0.05) {$\mu'_{\firstidx,\secondidx}$};
    \else
        % Light gray bar (height 1-β)
        \filldraw[fill=gray!30, draw=black] (\x, \p2start) rectangle (\x+\barwidth, \p2start+0.6);
        \node[font=\tiny, rotate=90, anchor=east] at (\x+\barwidth/2, \p2start-0.05) {$\mu'_{\firstidx,\secondidx}$};
    \fi
}

% Y-axis for p_{2,1} chart
\draw[->] (-0.5, \p2start) -- (-0.5, \p2start+2.5);
\node[left] at (-0.5, \p2start) {0};
\node[left] at (-0.5, \p2start+0.6) {$\frac{1-\beta}{2kN}$};
\node[left] at (-0.5, \p2start+1.4) {$\frac{1+\beta}{2kN}$};
\node[left] at (-0.5, \p2start+2) {$\frac{1}{kN}$};

% Labels - centered at bottom of each sub-plot
\pgfmathparse{(\leftstart+\rightstart+\halfwidth)/2}
\edef\centerx{\pgfmathresult}
\node[below] at (\centerx, \qstart-0.8) {$q_1$};
\node[below] at (\centerx, \qtwostart-0.8) {$q_2$};
\node[below] at (\centerx, \pstart-0.8) {A $p_{1,l}$ from $\cD_1$};
\node[below] at (\centerx, \p2start-0.8) {A $p_{2,l}$ from $\cD_2$};
\end{tikzpicture}
\caption{Illustration of the distribution construction for $k=5$ and $N=3$.}
\label{fig:bousquet}
\end{figure}

We define two distributions $\cD_1$ and $\cD_2$: $\cD_1$ uniformly chooses from $\{p_{1,l} | l \in [k]^{[2N]}\}$ and $\cD_2$ uniformly chooses from $\{p_{2,l} | l \in [k]^{[2N]}\}$. The following claim comes directly from the information-theoretic part of Theorem 19 in \cite{bousquet2019optimal}. We reproduce the argument here for clarity.

\begin{lemma}\label{lmm:bkmconstruction}
Let $\cF$ be a bounded function class. Fix $N \in \NN,\beta \in \NN$ such that $\frac{1+\beta}{\beta} \in \NN$, and let $k = \frac{1+\beta}{\beta}$. Then, construct $q_1, q_2, \cD_1$ and $\cD_2$ as above. For any evaluation algorithm $\cA$,  there exists a $q^\star\in \cD_1 \cup \cD_2$ such that, given a sample $\sev$ of size at most $\sqrt{2kN \beta}$ from $q^\star$,  with probability at least $1/3$, $\hat q = \cA(\{q_1,q_2\},\sev)$ will be the wrong hypothesis, i.e., $d_\cF(\hat q, q^\star) = \max\{d_\cF(q_1, q^\star), d_\cF(q_2, q^\star)\}$.
\end{lemma}

\begin{proof}
We can view $\cA$ as a distinguisher between $\cD_1$ and $\cD_2$. In particular, if the underlying distribution
belongs to $\cD_1$, then $\cA$ errs exactly when, given $\sev$, $\cA(\{q_1,q_2\},\sev)=q_2$; if the
underlying distribution belongs to $\cD_2$, then $\cA$ errs exactly when $\cA(\{q_1,q_2\},\sev)=q_1$. Therefore, it suffices to lower bound the error probability of
distinguishing $\cD_1$ from $\cD_2$.

By \cref{lem:bkm22},

$$
\Pr[\cA \text{ errs}]
\geq
\frac12\left(1-TV(\cD_1^{\oplus m},\cD_2^{\oplus m})\right),
$$

where $\cD_i^{\oplus m}$ denotes the distribution obtained by first drawing
$p\sim \cD_i$ and then drawing $m$ independent samples from $p$. Hence it
remains to show that
$$
TV(\cD_1^{\oplus m},\cD_2^{\oplus m})\leq \frac13.
$$
For each $i\in[2N]$, let
$$
B_i=\bigcup_{j=1}^k \operatorname{supp}(\mu_{i,j})
$$
be the $i$-th big block. By construction, every distribution
$p\in \cD_1\cup \cD_2$ assigns mass exactly
$$
p(B_i)=\frac{1}{2N}
$$
to each big block $B_i$. Let $E$ be the event that the $m$ samples fall in $m$
distinct big blocks. Since the big-block marginal is uniform over $[2N]$ under
every $p\in \cD_1\cup \cD_2$, we have, under both $\cD_1^{\oplus m}$ and
$\cD_2^{\oplus m}$,
$$
\Pr(E^c)
\leq
\binom{m}{2}\frac{1}{2N}
\leq
\frac{m^2}{4N}.
$$
We next show that
$$
(\cD_1^{\oplus m})|E=(\cD_2^{\oplus m})|E.
$$
Conditioned on $E$, each occupied big block contains exactly one sample. The
occupied big blocks have the same law under $\cD_1^{\oplus m}$ and
$\cD_2^{\oplus m}$, because every distribution in $\cD_1\cup \cD_2$ gives mass
$1/(2N)$ to each big block.

It remains to check the conditional distribution inside an occupied big block.
Fix a block $B_i$, and write
$$
\bar\mu_i=\frac1k\sum_{j=1}^k \mu_{i,j}.
$$
We claim that, after averaging over the random selector $\ell(i)$, the
conditional distribution of the unique sample inside $B_i$ is $\bar\mu_i$,
regardless of whether the underlying family is $\cD_1$ or $\cD_2$.

Indeed, consider a left block $i\leq N$. For $p_{1,\ell}$, conditional on
landing in $B_i$, the selected component $\mu_{i,\ell(i)}$ has conditional
weight $2/k$, while each unselected component has conditional weight
$(1-\beta)/k$. Since $\ell(i)$ is uniform on $[k]$, the averaged conditional
weight of any fixed component $\mu_{i,j}$ is
$$
\frac1k\cdot \frac{2}{k}
+
\frac{k-1}{k}\cdot \frac{1-\beta}{k}
=
\frac{2+(k-1)(1-\beta)}{k^2}
=
\frac1k,
$$
where we used $k=(1+\beta)/\beta$. Thus the averaged conditional distribution
is $\bar\mu_i$.

For $p_{2,\ell}$ on the same left block, the selected component has conditional
weight $0$, while each unselected component has conditional weight
$(1+\beta)/k$. Hence the averaged conditional weight of any fixed component is
$$
\frac1k\cdot 0
+
\frac{k-1}{k}\cdot \frac{1+\beta}{k}
=
\frac{(k-1)(1+\beta)}{k^2}
=
\frac1k.
$$
Thus the averaged conditional distribution is again $\bar\mu_i$. The
calculation for right blocks $i>N$ is identical, with the roles of $\cD_1$ and
$\cD_2$ reversed.

Because $E$ ensures that no big block contributes more than one sample, there
is no dependence coming from reusing the same hidden selector $\ell(i)$ twice
within a block. Across different big blocks, the selectors $\ell(i)$ are
independent. Therefore the full conditional joint distributions agree:
$$
(\cD_1^{\oplus m})|E=(\cD_2^{\oplus m})|E.
$$
Applying \cref{lem:bkm23} with this event $E$, in the form
$$
TV(P,Q)\leq TV(P|E,Q|E)+2P(E^c)+2Q(E^c),
$$
gives
$$
TV(\cD_1^{\oplus m},\cD_2^{\oplus m})
\leq
0+2\cD_1^{\oplus m}(E^c)+2\cD_2^{\oplus m}(E^c)
\leq
4\cdot \frac{m^2}{4N}
=
\frac{m^2}{N}.
$$
Choosing $N\geq 3m^2$, we obtain
$$
TV(\cD_1^{\oplus m},\cD_2^{\oplus m})\leq \frac13.
$$
Therefore \cref{lem:bkm22} implies
$$
\Pr[\cA \text{ errs}]
\geq
\frac12\left(1-\frac13\right)
=
\frac13.
$$
Equivalently, there exists some
$$
q^\star\in \{p_{1,\ell}:\ell\in[k]^{[2N]}\}\cup \{p_{2,\ell}:\ell\in[k]^{[2N]}\}
$$
such that, with probability at least $1/3$ over
$S_{\mathrm{eval}}\sim (q^\star)^m$, the evaluation algorithm selects the hypothesis
associated with the wrong family.
\end{proof}

% \begin{restatable}{prop}{evallowerbound}
% \label{thm:evallowerbound}
% Consider a real-valued class $\cF$ for which there exists $\gamma \in (0,1/2)$ such that $\cF$ $\gamma$-shatters arbitrarily large finite sets. Then, for any $\tau > 0$ and $m \in \NN$ there exists a pair $\cH = \{q_1,q_2\}$ such that for every score function $s$, there exists a $q^\star$ such that with probability at least $1/3$,
% \[
% d_\cF(\hat q, q^\star) \geq (3-\tau)\min\{d_\cF(q_1,q^\star),d_\cF(q_2,q^\star)\},
% \]
% where $\hat q = \argmin_{q\in \cH} s(q,\sev)$. 
% \end{restatable}

\subsection{Proof for \cref{thm:realeval}}

\begin{lemma}\label{lem:lbreal}
Let $\cF\subset[0,1]^\cX$ be a bounded function class, and denote
\[
\gamma^\star=\inf\{\gamma\;:\; \pd_\gamma(\cF)<\infty\}.
\]
Assume that $\gamma^\star>0$. Then, for any
$\gamma^{-},\gamma^{+}$, such that $\gamma^+>\gamma^\star>\gamma^-$, and any $n \in \mathbb{N}$, there exists $r\in (0,1)$, a set $X$ and $n$ subsets $X_1, \ldots, X_{n}\subset X$ such that:
    \begin{enumerate}
        \item The $X_i$'s are pairwise disjoint.
        % \item The union $\bigcup_i X_i$ is $\gamma^*$-fat-shattered by $\cF$.
        \item  The union of the sets $\bigcup_{i=1}^n X_i$ is $\gamma^{-}$-shattered at the same threshold $r$ by $\cF$.
        \item Let $P_i$ be the uniform distribution over $X_i$. Then for every $i, j$ and every $\phi \in \cF$ 
        \[
        |\EEs{P_i}{\phi} - \EEs{P_j}{\phi}| \leq \gamma^+.
        \]
    \end{enumerate}
\end{lemma}
\begin{proof}[Proof of \cref{lem:lbreal}]
Choose some $\bar{\gamma}$ such that $\gamma^+ > \bar{\gamma} > \gamma^\star$. It follows that $\pd_{\bar{\gamma}}(\cF) < \infty$. Let $\eta = 2(\gamma^+ - \bar{\gamma})$, and
let $m, K>0$ to be chosen later. By \cref{lem:fixedthreshold}, there is a set $X$ of size $K$ that is $\gamma^{-}$-shattered by $\cF$ at the same threshold. We generate the sets $X_i$ by sampling $m$ points from $q$, the uniform distribution over $X$. Note that by definition condition $(2.)$ will hold. For condition $(3.)$ we use the same symmetrization argument as in \cref{lem:uniform-convergence-bound} (see also Lemma 16 in \cite{BARTLETT1998}) to deduce that the probability it will fail for  any fixed pair $(i,j)$ is bounded by
\[\Pr_{X_i,X_j \sim q^m}\left[
\sup_{f \in \cF}
\left|
\frac{1}{m}\sum_{x \in X_i} f(x) - \frac{1}{m}\sum_{x \in X_j} f(x)
\right|
> \bar{\gamma} + \frac{\eta}{2}
\right] \leq 2N\left(\cF,\bar{\gamma}/2+\eta/8,2m\right)\exp\left(\frac{-\eta^2m}{32} \right).\]
So by the union bound, item $(3.)$ will fail with probability at most 
\[\binom{n}{2}2N\left(\cF,\bar{\gamma}/2+\eta/8,2m\right)\exp\left(\frac{-\eta^2m}{32} \right),\]
Applying \cref{lem:covering} we see that the above probability goes to $0$ as $m$ goes to infinity, hence we may pick some $m$ for which it is at most $\frac{1}{4}$.

Now for the first item, note that the probability of two sampled points colliding is at most \[\binom{2nm}{2}\frac{1}{|X|}.\] 
Hence if we take $K$ the size of $X$ to be larger than $4\binom{2nm}{2}$, the above will happen with probability at most $\frac{1}{4}$. Thus by union bound all the conditions will hold with positive probability, implying the existence of such sets.
\end{proof}

\realeval*
\begin{proof}
The proof follows immediately from \cref{prop:fateval}, with $\pd_{\gamma}(\cF)$ underpinning the dichotomy.
\end{proof}
\fateval*
\begin{proof}\label{proof:fateval}
When $\pd_{\gamma}(\cF)<\infty$ for all $\gamma > 0$, from Proposition 3.4 of \cite{aiyer2026theoretical}, $d_{\cF}$ is estimable (and hence strongly evaluable) with sample complexity \[m(\epsilon,\delta) = O\left(\frac{1}{\epsilon^2}d\log^2 \frac{4d}{\epsilon^2} + \log \frac{1}{\delta} \right)\,,\] where $d = \pd_{\epsilon/24}(\cF)$. This proves the first item. 

Now, consider the case where there exists $\gamma \in (0,1/2)$ such that $\pd_{\gamma}(\cF)=\infty$. We know again from Proposition 3.4 of \cite{aiyer2026theoretical} that $d_{\cF}$ is 3-weakly evaluable. 
It remains to show that for any $c'<3$, $d_{\cF}$ is not $c'$-weakly evaluable.

For any $\beta\in (0,1)$ and a collection of elements $\bmu$, following the construction and guarantee by \cref{lmm:bkmconstruction}, we can obtain a hypothesis class $\cH = \{q_1,q_2\}$ and two distribution families $\cD_1=\{p_{1,l}|l:[2N]\mapsto [k]\}$ and $\cD_2=\{p_{2,l}|l:[2N]\mapsto [k]\}$ such that no evaluation algorithm can evaluate $\cH$.
Fix any $\tau > 0$. The only thing that remains is to find a collection of elements $\bmu$ such that $d_\cF(q_2, p_{1,l}) \geq  (3-\tau) d_\cF(q_1, p_{1,l})$ for an index $l$.

Let $k = \frac{1+\beta}{\beta}$.
 Let $\gamma^\star= \inf\{\gamma>0 : \pd_\gamma(\cF)<\infty\}$. For a small $\alpha > 0$, let $\gamma^+ = \gamma^\star + \alpha$ and $\gamma^- = \gamma^\star - \alpha$. Now, choose $\beta$ so that $\frac{3-\beta}{1+\beta}\cdot \frac{\gamma^-}{\gamma^+} \geq 3-\tau$.

By \cref{lem:lbreal}, we can find a set $X$ and a sequence of $2kN$ disjoint subsets $\{X_{i,j}|i\in [2N],j\in [k]\}$ such that  $\bigcup X_{i,j}$ is $\gamma^{-}$-shattered by $\cF$ at the same threshold and $ |\EEs{\mu_{i,j}}{\phi} - \EEs{\mu_{i',j'}}{\phi}| \leq \gamma^+$ for every pair $(i,j),(i',j')\in [2N]\times [k]$ and every $\phi \in \cF$, where $\mu_{i,j}$ is the uniform distribution over $X_{i,j}$. 
 
For any index function $l$, we have
\[d_\cF(q_1, p_{1,l}) = \frac{1+\beta}{2kN}\sup_{\phi\in \cF} \sum_{i=1}^N (\EEs{\mu_{i,l(i)}}{\phi}) - \EEs{\mu_{i+N,l(i+N)}}{\phi}))\,.\]
Note the function within the supremum is a weighted sum of $\EEs{\mu_{i,l(i)}}{\phi}$'s and $\EEs{\mu_{i+N,l(i+N)}}{\phi}$'s. We define the normalized weighted sum $a(\bmu,l)$ as 
\[a(\bmu,l) := \frac{2k}{1+\beta} \cdot d_\cF(q_1, p_{1,l}) = \sup_{\phi\in \cF} \frac{1}{N}\sum_{i=1}^N (\EEs{\mu_{i,l(i)}}{\phi}) - \EEs{\mu_{i+N,l(i+N)}}{\phi})) \,.\]

We also have 
\[d_\cF(q_2, p_{1,l}) =\sup_{\phi\in \cF} \frac{1}{2kN}\sum_{i=1}^N ( (1-\beta)(\EEs{\mu_{i,l(i)}}{\phi}) - \EEs{\mu_{i+N,l(i+N)}}{\phi})) + 2\beta \sum_{j \neq l(i)}(\EEs{\mu_{i+N,j}}{\phi} - \EEs{\mu_{i,j}}{\phi}))\,.\]
Given $\bmu$ and $l$, we define the normalized weighted sum $b(\bmu,l)$ as 

\begin{align*}
    b(\bmu,l) :=& \frac{2k}{3-\beta} \cdot d_\cF(q_2, p_{1,l})\\
=& \sup_{\phi\in \cF} \frac{1}{(3-\beta)N}\sum_{i=1}^N ( (1-\beta)(\EEs{\mu_{i,l(i)}}{\phi}) - \EEs{\mu_{i+N,l(i+N)}}{\phi})) + 2\beta \sum_{j \neq l(i)}(\EEs{\mu_{i+N,j}}{\phi} - \EEs{\mu_{i,j}}{\phi})) \,.
\end{align*}

We can show that $a(\bmu,l) \leq \gamma^+$ and $b(\bmu,l) \geq \gamma^-$. By \cref{lem:lbreal}, we have that $a(\bmu,l) \leq \gamma^+$ for all $l \in [k]^{[2N]}$. Since $\bigcup X_{i,j}$ can be $\gamma^{-}$-shattered by $\cF$ at the same threshold, we also have 
\[b(\bmu,l) \geq \gamma^-\,.\]

Combining the two bounds, we obtain
\[
\frac{d_{\cF}(q_2,p_{1,l})}{d_{\cF}(q_1,p_{1,l})}
\ge
\frac{3-\beta}{1+\beta}\cdot \frac{\gamma_-}{\gamma_+}
\ge
3-\tau.
\]
Thus
\[
d_{\cF}(q_2,p_{1,l})\ge (3-\tau)\, d_{\cF}(q_1,p_{1,l}).
\]
So conditional on the event that the evaluation algorithm returns the wrong hypothesis, we have
\[
d_{\cF}(q^\star,\hat q)
\ge
(3-\tau)\min\ac{d_{\cF}(q^\star,q_1),\, d_{\cF}(q^\star,q_2)}.
\]

Now, for any $c'<3$, if we take $\tau$ such that $3-\tau > c'$, by \cref{lmm:bkmconstruction}, we have that $d_{\cF}$ is not $c'$-weakly evaluable, proving that the scale $c=3$ is optimal.
\end{proof}

\section{Acknowledgements}
Yishay Mansour received funding from the European Research Council (ERC) under the European Union’s Horizon 2020 research and innovation program (grant agreement No. 882396), by the Israel Science Foundation,  the Yandex Initiative for Machine Learning at Tel Aviv University and a grant from the Tel Aviv University Center for AI and Data Science (TAD).

Shay Moran is a Robert J.\ Shillman Fellow; he acknowledges support by Israel PBC-VATAT, by the Technion Center for Machine Learning and Intelligent Systems (MLIS), and by the the European Union (ERC, GENERALIZATION, 101039692). Views and opinions expressed are however those of the author(s) only and do not necessarily reflect those of the European Union or the European Research Council Executive Agency. Neither the European Union nor the granting authority can be held responsible for them.

Han Shao acknowledges support from
an Adobe Research gift.

\newpage

\bibliographystyle{alpha}
\bibliography{ref}

\newpage
\appendix

\end{document}